\documentclass[10pt, a4paper]{scrartcl}

\usepackage{latexsym}
\usepackage{amssymb}
\usepackage[nonamelimits]{amsmath}
\usepackage[authoryear, round]{natbib}
\usepackage{url}
\usepackage{amsthm}

\usepackage{color}
\definecolor{myred}{rgb}{0.5,0,0}
\definecolor{myblue}{rgb}{0,0,0.75}
\definecolor{mygreen}{rgb}{0,0.5,0}

\usepackage{ifpdf}

\ifpdf
   \usepackage[pdftex]{graphicx}
   \usepackage[pdftex,
              pdftitle={Calibrating sufficiently},
               pdfsubject={},
               pdfkeywords={},
               pdfauthor={Dirk Tasche},
               pdfstartview=FitR,
               breaklinks=true,
               colorlinks=true,
               citecolor=myred,
               linkcolor=myblue,
               urlcolor=mygreen]{hyperref}
\else
   \usepackage[dvips]{graphicx}
   \usepackage{hyperref}
   \usepackage{rotating}
\fi

\newtheorem{theorem}{Theorem}[section]
\newtheorem{lemma}[theorem]{Lemma}
\newtheorem{remark}[theorem]{Remark}
\newtheorem{example}[theorem]{Example}
\newtheorem{proposition}[theorem]{Proposition}
\newtheorem{definition}[theorem]{Definition}
\newtheorem{corollary}[theorem]{Corollary}
\newtheorem{assumption}[theorem]{Assumption}

\newcommand{\var}{\mathrm{var}}

\title{Calibrating sufficiently}

\author{%
Dirk Tasche\thanks{Swiss Financial Market Supervisory Authority (FINMA). The opinions 
expressed in this note are those of the author 
and do not necessarily reflect views of FINMA.\newline
Email: dirk.tasche@gmx.net}}

\date{}

\begin{document}
\maketitle

\begin{abstract}
When probabilistic classifiers are trained and calibrated, the so-called grouping loss
component of the calibration loss can easily be overlooked. Grouping loss 
refers to the gap between observable information and information actually exploited 
in the calibration exercise. We investigate the relation between grouping loss and
the concept of sufficiency, identifying comonotonicity as a useful criterion for
sufficiency. We revisit the probing reduction approach of 
\cite{langford2005estimating}
and find that it produces an estimator of probabilistic classifiers that reduces
grouping loss. Finally, we discuss Brier curves as tools to support 
training and `sufficient' calibration of probabilistic classifiers.
\end{abstract}

\section{Introduction}
\label{se:intro}

Binary classification, in the first place, deals with decision tools (classifiers) 
that facilitate the prediction of the classes of instances on the basis of the so-called features
of the instances. Accordingly, the simplest classifiers are crisp (or discrete) in the sense of
having the set $\{0,1\}$ as output range: $1$ for `predict positive class', $0$ for `predict
negative class'. Scoring (or soft) classifiers provide output in a continuous range,
usually with the interpretation that high values indicate high likelihood of the
instance belonging to the positive class, while low values suggest that membership of
the negative class is more likely.

In many applications of classification, there is a need for `calibrated' probabilistic classifiers
which reflect the likelihood of the positive class given the features of an instance in a
frequentist statistical sense \citep{Platt99probabilisticoutputs, Zadrozny&Elkan2002,
Cohen&Goldszmidt, kull2017betacalibration}. How to best achieve
good calibration and how to measure it are active research areas 
\citep{boken2021appropriateness, roelofs2020mitigating}.

In this paper, we argue that for attaining good calibration of a probabilistic
classifier the so-called `grouping loss' \citep{kull2014reliability} must not be 
neglected. We show that grouping loss reflects an information gap that
is caused by not fully exploiting the information provided by the features when
learning the probabilistic classifier.

Most of the bits and pieces needed for a thorough discussion of grouping loss are available
in the literature:
\begin{itemize} 
\item Sufficiency and admissibility (see \cite{devroye1996probabilistic}, Section~32.3)
\item Probing reduction \citep{langford2005estimating}
\item Brier curves \citep{hernandez2011brier}
\end{itemize}
In this paper, we put these pieces together and explain how they complement each other. 
The paper, in particular, presents the following new contributions to the subject:
\begin{itemize}
\item We point out that the common practice of first developing a so-called scoring
classifier and then calibrating it to obtain a probabilistic classifier risks loss of accuracy 
(as measured by the Brier score)
unless the scoring classifier is strongly comonotonic with the posterior probability of the
positive class (Proposition~\ref{pr:comon}).
\item We extend the findings on sufficiency by \cite{brocker2009reliability} 
(Proposition~\ref{pr:broecker}) and present a counter-example to the `if'-part of Theorem~10.2 of 
\cite{schervish1989general} (Example~\ref{ex:counter}).
\item We provide a more general representation of the `probing predictor' of  
\cite{langford2005estimating} as well as a rigorous proof of their Theorem~2 (Theorem~\ref{th:Langford}
below).
\item We determine, in full generality, the point where the Brier curve of a calibrated probabilistic classifier
takes its maximum (Proposition~\ref{pr:properties}) and derive a novel lower bound 
for the Brier score associated with a binary regression problem (Proposition~\ref{pr:cov}).
\item Following \cite{HolzmannInformation2014}, we argue that the concept of 
$\sigma$-fields from measure theory is best-suited for
analysing the role of information in the calibration of probabilistic classifiers. 
\end{itemize}

We set the stage for the paper in Section~\ref{se:stage} where we also state more formally the problem.
In Sections~\ref{se:sufficiency}, \ref{se:probing} and \ref{se:curves} respectively, then
the concept of sufficiency, probing reduction, and Brier curves are discussed in detail.
Section~\ref{se:recom} provides a recommendation of how to avoid or at least control
grouping loss, or in other words how to achieve `sufficient calibration'. 
The paper concludes with a summary of the findings in Section~\ref{se:conclusions}.
Appendix~\ref{se:proofs}
provides proofs for some of the results of the paper.


\section{Setting the stage}
\label{se:stage}

We follow the examples of \cite{reid2011information}, \cite{zhao2013beyond} and others
in studying the subject of this paper at population level -- as opposed to sample level. 
For the reason to do
so, see \cite{zhao2013beyond} who state in Section~2.2 `To remove 
finite sample effects, we consider a data set of infinite size and hence the asymptotic
expressions of error rate and balanced error rate.' 
Refer also to Appendix~B.1 of \cite{zhao2013beyond} for their rationale to deploy measure-theoretic
formalism to classification problems.

\subsection{Notation and setting}
\label{se:notation}

We use the following population-level description of the binary classification problem.
See standard textbooks on probability theory like \cite{BauerProbTheory} for
formal definitions and background of the notions introduced in Assumption~\ref{as:setting}.
\begin{assumption}\label{as:setting}
$(\Omega, \mathcal{A}, P)$ is a probability space. An event $A \in \mathcal{A}$ with $0 < P[A] < 1$ and
a sub-$\sigma$-field $\mathcal{H} \subset \mathcal{A}$ with $A \notin \mathcal{H}$ are fixed.
\end{assumption}
\noindent\textbf{Interpretation}
\begin{itemize} 
\item The elements $\omega$ of $\Omega$ are objects (or instances) with class and feature properties. 
$\omega \in A$ means that $\omega$ belongs to class 1 (or the positive class). 
$\omega \in A^c = \Omega\setminus A$ 
means that $\omega$ belongs to class 0 (or the negative class).
\item The $\sigma$-field $\mathcal{A}$ of events $M \in \mathcal{A}$ is a collection of subsets $M$
of $\Omega$ with the property that they can be assigned probabilities $P[M]$ in a logically consistent way.
\item Binary classification problem: The sub-$\sigma$-field $\mathcal{H} \subset \mathcal{A}$ 
contains the events which are observable
at the time when the class of an object $\omega$ has to be predicted. Since $A \notin \mathcal{H}$, then
the class of an object is not known. It can only be inferred on the basis of the events $H \in \mathcal{H}$
which are assumed to reflect the features of the instance.
\end{itemize}

Reconciliation with the standard setting for binary classification in the machine learning and pattern 
recognition literature (see e.g.~\citealp{devroye1996probabilistic}): 
\begin{itemize} 
\item Typically a random vector $(X, Y)$ is studied, where $X$ stands for the features of an instance 
and $Y$ stands for its class. $X$ is assumed to take values in a feature space $\mathfrak{X}$ (often
$\mathfrak{X} = \mathbb{R}^d$) while 
$Y$ takes either the value 0 (or $-1$) or the value 1 (for the positive class).
\item Standard formulation of the binary classification problem: 
Infer the value of $Y$ from $X$ or make
an informed decision on the occurence or non-occurrence of $Y=1$ despite only being able 
to observe the values of $X$.
\item This is captured by the measure-theoretic setting of Assumption~\ref{as:setting}: 
Assume that $X$ and $Y$ map $\Omega$ into $\mathfrak{X}$ and $\{0,1\}$ respectively. 
Choose $\mathcal{H} = \sigma(X)$ (the
smallest sub-$\sigma$-field of $\mathcal{A}$ such that $X$ is measurable) and 
$A = \{Y=1\} = \{\omega\in\Omega: Y(\omega)=1\}$.
\item Often no probability space is specified but only a sample $(x_1, y_1), \ldots, (x_n, y_n)$ of
realisations of $(X,Y)$ is
considered. Usually this sample is assumed to have been generated by i.i.d.\ drawings from
some population distribution which may be identified with $(\Omega, \mathcal{A}, P)$ as described above.
\end{itemize}

\subsection{Posterior class probabilities}

The probability of $A$ conditional on $\mathcal{H}$ (or posterior probability of the positive class) 
$\Psi = P[A\,|\,\mathcal{H}]$ 
is a central concept for binary classification and related areas, for the following reasons:
\begin{itemize} 
\item As a consequence of Proposition~\ref{pr:brier} below on the decomposition of the Brier Score
into refinement and calibration loss, $\Psi$ is the best predictor of $A$ in the sense of 
minimising the mean squared error.
\item `Thresholding' $\Psi$ produces the Bayes classifiers which minimise the Bayes error and cost-sensitive
Bayes error respectively (\citealp{vanTrees}, p.~27; 
\citealp{devroye1996probabilistic}, Theorem~32.4).
\item More generally, thresholding $\Psi$ also represents the best basis for many decision criteria 
for label assignments (classification) based on the confusion matrix
    (\citealp{koyejo2014consistent},  Theorem~2).
\item $\Psi$ is best for statistical tests and binary classification based on Neyman-Pearson criteria 
    \citep{Scott2019}.
\item $\Psi$ is an optimum solution to the bipartite ranking problem (\citealp{menon2016bipartite}, Proposition~63).
\item From the perspective of practice, estimates of $\Psi$ may serve as plug-in classifiers
(\citealp{devroye1996probabilistic}, Section~2.5).
\end{itemize}
We refer to Section~4.1 of \cite{Durrett} for the formal definitions and properties of
\begin{itemize}
\item expectation $E[Z\,|\,\mathcal{H}]$ of a random variable $Z$ conditional on $\mathcal{H}$, and
\item probability $P[A\,|\,\mathcal{H}]$ of $A$ conditional on $\mathcal{H}$ respectively.
\end{itemize}
In the machine learning literature, often the term \emph{posterior class probability} rather than
conditional probability is used to refer to $P[A\,|\,\mathcal{H}]$, as opposed to the \emph{prior
probability} $P[A]$ which in our setting would rather be called unconditional probability of $A$.
For a more concise notation, in the following we will sometimes write $\Psi$ or $\Psi_{\mathcal{H}}$
for $P[A\,|\,\mathcal{H}]$.
\begin{remark}[$\Psi$ vs.\ $\eta$]\emph{%
In the machine learning and pattern recognition literature, the posterior class probability often is
denoted by $\eta$. See for instance Section~2.1 of \cite{devroye1996probabilistic} who
introduced it as `regression of $Y$ on $X$'. If the sub-$\sigma$-field $\mathcal{H}$ from
Assumption~\ref{as:setting} is generated by a features random vector $X$ with values
in a measurable space $\mathfrak{X}$, i.e.\ $\mathcal{H}
= \sigma(X)$, then the factorisation lemma (\citealp{BauerProbTheory}, Section~15) 
implies $\Psi = P[A\,|\,\sigma(X)] = \eta(X)$
for some measurable function $\eta: \mathfrak{X} \to [0,1]$. This function $\eta$ is 
the posterior probability (or regression function) \citep{devroye1996probabilistic}
and many other authors are referring to. In this paper, the slightly different
$\Psi$-concept
is preferred because it helps avoiding frequent changes of variables in expectations and
integrals.
} \hfill \qed
\end{remark}

\subsection{The Brier score and its components}

The Brier score is an important tool for the analysis and practice of 
binary classification:
\begin{itemize} 
\item As shown below, the Brier score plays an important role in the theoretical
analysis of estimators of the posterior class probability $\Psi$.
\item By its definition as a mean squared error, representations of the Brier score and its
components sometimes can be directly deployed for parametric or non-parametric regression. 
\end{itemize}
See Chapter~6 of \cite{Hand97} for further background information on the Brier score.

\begin{definition}[Brier score]\label{de:brier}
Under Assumption~\ref{as:setting}, the \emph{Brier score} $BS(Z)$ (with 
respect to the event $A$) of a square-integrable $\mathcal{A}$-measurable
random variable $Z$ is defined as
$$BS(Z) \ = \ E[(\mathbf{1}_A - Z)^2],$$
where $\mathbf{1}_A$ with $\mathbf{1}_A(\omega) = 
\begin{cases} 1, & \omega \in A\\ 0, & \omega \not\in A
\end{cases}$ denotes
the indicator function of $A$.
\end{definition}

\begin{proposition}[Decomposition of Brier score]\label{pr:brier}
Denote by $\Psi_{\mathcal{H}}$ the probability of $A$ conditional on $\mathcal{H}$. Assume $Z$ is 
a square-integrable $\mathcal{H}$-measurable random variable (interpreted as estimator of $A$ or
\emph{probabilistic classifier}). 
Then $BS(Z)$ can be represented as sum of \emph{refinement loss} 
$E[\Psi_{\mathcal{H}}\,(1-\Psi_{\mathcal{H}})]$
and \emph{calibration loss} $E[(\Psi_{\mathcal{H}} - Z)^2]$:
\begin{subequations}
\begin{equation}\label{eq:decomp}
BS(Z) \ = \ E[\Psi_{\mathcal{H}}\,(1-\Psi_{\mathcal{H}})] + 
    E[(\Psi_{\mathcal{H}} - Z)^2].
\end{equation}
Assume $\mathcal{G}$ is another sub-$\sigma$-field of $\mathcal{A}$ such that
$\mathcal{G}\subset\mathcal{H}$ and define $\Psi_{\mathcal{G}} = P[A\,|\,\mathcal{G}]$. 
If $Z$ is $\mathcal{G}$-measurable then decomposition~\eqref{eq:decomp} can be 
modified to
\begin{equation}\label{eq:refined}
BS(Z) \ = \ E[\Psi_{\mathcal{H}}\,(1-\Psi_{\mathcal{H}})] + E[(\Psi_{\mathcal{H}} - \Psi_{\mathcal{G}})^2]
    + E[(\Psi_{\mathcal{G}} - Z)^2].
 \end{equation} 
In particular, it holds that
\begin{equation}\label{eq:grouping}
E[\Psi_{\mathcal{G}}\,(1-\Psi_{\mathcal{G}})] \ = \ E[\Psi_{\mathcal{H}}\,(1-\Psi_{\mathcal{H}})] 
    + E[(\Psi_{\mathcal{H}} - \Psi_{\mathcal{G}})^2].
\end{equation}
The term $E[(\Psi_{\mathcal{H}} - \Psi_{\mathcal{G}})^2]$ is called \emph{grouping loss}.
\end{subequations}
\end{proposition}
\begin{proof} Regarding \eqref{eq:decomp} a short computation yields
\begin{align*}
E[(\mathbf{1}_A - Z)^2] & \ = \ E[(\mathbf{1}_A - \Psi_{\mathcal{H}})^2] + 
    2\,E[(\mathbf{1}_A - \Psi_{\mathcal{H}})\,(\Psi_{\mathcal{H}}-Z)] +
    E[(\Psi_{\mathcal{H}}-Z)^2]\\
& \ = \  P[A] - 2\,E[\mathbf{1}_A\,\Psi_{\mathcal{H}}] + E[\Psi_{\mathcal{H}}^2] \\
& \qquad + 2\,E\bigl[(P[A\,|\,\mathcal{H}]-\Psi_{\mathcal{H}})\,(\Psi_{\mathcal{H}}-Z)\bigr] 
    + E[(\Psi_{\mathcal{H}}-Z)^2]\\
& \ = \ P[A] - E[\Psi_{\mathcal{H}}^2] + E[(\Psi_{\mathcal{H}}-Z)^2]\\
& \ = \ E[\Psi_{\mathcal{H}}\,(1-\Psi_{\mathcal{H}})] + 
    E[(\Psi_{\mathcal{H}} - Z)^2].
\end{align*}
\eqref{eq:refined} follows by straightforward application of the tower property of conditional expectation:
\begin{align*}
E[(\Psi_{\mathcal{H}} - Z)^2] & \ = \ E[(\Psi_{\mathcal{H}} - \Psi_{\mathcal{G}})^2] +
    2\,E\bigl[(\Psi_{\mathcal{H}} - \Psi_{\mathcal{G}})\,(\Psi_{\mathcal{G}} - Z)\bigr] +
    E[(\Psi_{\mathcal{G}} - Z)^2]\\
& \ = \  E[(\Psi_{\mathcal{H}} - \Psi_{\mathcal{G}})^2] + E[(\Psi_{\mathcal{G}} - Z)^2]\\
& \qquad + 2\,E\bigl[(E[\Psi_{\mathcal{H}}\,|\,\mathcal{G}] - \Psi_{\mathcal{G}})\,
    (\Psi_{\mathcal{G}} - Z)\bigr]. 
\end{align*}
\eqref{eq:grouping} immediately follows from combining \eqref{eq:decomp} and \eqref{eq:refined}.
\end{proof}

Some comments on Proposition~\ref{pr:brier}:
\begin{itemize} 
\item Proposition~\ref{pr:brier} implies that the probability of $A$ conditional on $\mathcal{H}$ is
the best estimate of $A$ given the information reflected by $\mathcal{H}$, in the sense of minimising
the mean-squared error. Thus Proposition~\ref{pr:brier} provides an alternative characterisation
of conditional probabilities.
\item \cite{Hand97} dealt with
the multi-class case of \eqref{eq:decomp} and called
$E\bigl[(\mathbf{1}_A - Z)^2\bigr] + E\bigl[(\mathbf{1}_{A^c} - (1-Z))^2\bigr]
= 2\,E\bigl[(\mathbf{1}_A - Z)^2\bigr]$ \emph{Brier inaccuracy}. 
\item In the context
of binary classification, the Brier Score
is sometimes simply referred to as \emph{mean squared error} (\citealp{hernandez2012unified}, Definition~8). 
\item \cite{Hand97}, Section~6.5, called the term $E[\Psi\,(1-\Psi)]$ \emph{inseparability} and
the term $E[(\Psi - Z)^2]$ \emph{imprecision}.
\item We follow \cite{hernandez2012unified}, p.~2841, with the terms
refinement loss for $E[\Psi\,(1-\Psi)]$ and
calibration loss for $E[(\Psi - Z)^2]$ respectively.
\item \cite{kull2014reliability}, p.~22, introduced the concept of 
`grouping loss'. With a view on the fact that the $\sigma$-fields $\mathcal{H}$ and $\mathcal{G}$
reflect information that is utilised for estimating the respective conditional probabilities of
$A$, it would also seem appropriate to call the term $E[(\Psi_{\mathcal{H}} - \Psi_{\mathcal{G}})^2]$
\emph{information gap loss}. \cite{kull2014reliability} called the Brier score component
$E[(\Psi_{\mathcal{G}} - Z)^2]$ \emph{group-wise calibration loss}.
\item  Note the following alternative representation of the refinement loss:
\begin{equation}\label{eq:alternative}
E[\Psi\,(1-\Psi)] \ = \ P[A]\,(1-P[A]) - \var[\Psi]\ = \ \var[\mathbf{1}_A] - \var[\Psi].
\end{equation}
$\var[\mathbf{1}_A] = P[A]\,(1-P[A])$ is called \emph{uncertainty} while $\var[\Psi]$ is called \emph{resolution}.
\end{itemize}
Minimisation of the refinement loss in general
is a matter of feature selection which is comprehensively covered in the literature
\citep{chandrashekar2014features}.
Similarly, once feature selection has been completed and resulted in a fixed $\sigma$-field 
$\mathcal{H}$, reducing calibration loss by transforming scoring classifiers into estimators
of the true posterior probabilities given the scoring classifiers is a topic
treated extensively in the literature (for a survey see \citealp{kull2017betacalibration}, Section~1).

However, as already observed by \cite{Murphy&Winkler}, the fact that the group-wise 
calibration loss $E[(\Psi_{\mathcal{G}} - Z)^2]$ vanishes does not necessarily imply
that the entire calibration loss is zero because the grouping loss can still be positive.
\cite{kull2014reliability}, p.~22, advised to `train a new model'
in order to reduce the grouping loss portion of the calibration loss.
In this paper, we focus attention to the grouping loss and to some suggestions of how to avoid it
in the first place.

\subsection{Classifier calibration and grouping loss}
\label{se:calibration}

We study the connection between the following two approaches to the problem of estimating 
$\Psi = P[A\,|\,\mathcal{H}]$ under Assumption~\ref{as:setting} 
(cf.~also \citealp{menon2012predicting}, Sections 3.1 and 3.2):
\begin{enumerate} 
\item \label{it:direct} Directly estimate $P[A\,|\,\mathcal{H}]$ from the observed data. Most of the time, this is 
a hard problem requiring huge datasets,
primarily due to the curse of dimensionality if one wants to exploit the full information available in $\mathcal{H}$
since $\mathcal{H}$ typically represents observations in a multidimensional space
(see, for instance, \citealp{Hand97}, Chapter~5).
\item \label{it:uni} Find a \emph{scoring classifier}
(\citealp{hernandez2011brier}, Section~2.1) $S$, i.e.\ a real-valued 
$\mathcal{H}$-measu\-rable random variable, 
such that high values of $S$ reflect strong confidence that $A$ has occurred and low values reflect
weak confidence that $A$ has occurred.
Such random variables are sometimes also called \emph{confidence scores} in the literature
\citep{roelofs2020mitigating}.
Then estimate $\Psi_{\sigma(S)} = P[A\,|\,\sigma(S)]$. This is an easier task because it basically means to
perform regression on one real-valued variable. In the machine learning literature 
such estimation exercises are called \emph{calibration} of $S$ \citep{zadrozny2001obtaining}. 
\end{enumerate}

Approach~\ref{it:uni} actually means post-processing the result of a previous supervised learning exercise,
for instance learning a support vector machine (SVM) or any other binary classification method that outputs a
scoring classifier. Approach~\ref{it:uni} therefore refers to a 
procedure with two steps which, however, in general both are considered to be easier and more
efficient than approach~\ref{it:direct}.

Assume that feature selection has resulted in a fixed maximum amount of usable information, measured
by the $\sigma$-field $\mathcal{H}$. Then the refinement loss component in \eqref{eq:decomp}
is a constant and, at the same time, the theoretical minimum estimation error that can be achieved
when predicting the positive class event $A$ based on the information provided by $\mathcal{H}$.
Applying the direct approach~\ref{it:direct} is equivalent to trying to minimise the calibration
loss component $E[(\Psi_{\mathcal{H}} - Z)^2]$ in \eqref{eq:decomp}.

In contrast, the second step of approach~\ref{it:uni} -- the calibration of $S$ -- is better described
by \eqref{eq:refined}, with $\mathcal{G} = \sigma(S)$. The minimum achievable estimation error (refinement
loss) is the same as for approach~\ref{it:direct} but the calibration procedure is dealing only
with the calibration loss portion
\begin{equation*}
E[(\Psi_{\mathcal{G}} - Z)^2]\ = \ E[(\Psi_{\sigma(S)} - Z)^2].
\end{equation*}
Indeed, measuring the goodness of calibration in the shape of $E[(\Psi_{\sigma(S)} - Z)^2]$ is
an active research area (see \citealp{roelofs2020mitigating} for a recent example).

As mentioned before, \cite{kull2014reliability} identified the grouping loss 
as a factor that impacts the total calibration loss (`instance-wise calibration loss' in their words).
\cite{kull2014reliability} observed that
\begin{equation}\label{eq:condvar}
E[(\Psi_{\mathcal{H}} - \Psi_{\sigma(S)})^2] \ =\ E\bigl[\var[\Psi_{\mathcal{H}}\,|\,\sigma(S)]\bigr].
\end{equation}
They proposed a local regression approach for estimating $\var[\Psi_{\mathcal{H}}\,|\,\sigma(S)]$ 
(more precisely a so-called \emph{reliance map} based on this conditional variance) as a 
measure of reliability of their estimate of $\Psi_{\sigma(S)}$ as substitute of $\Psi_{\mathcal{H}}$.
This reliability measure then was used to improve an estimator for multi-class posterior
probabilities.

In the remainder of this paper, we 
\begin{itemize} 
\item revisit in Section~\ref{se:sufficiency}  the concept of sufficiency and its relation to
 the property that the grouping loss completely vanishes,
\item revisit in Section~\ref{se:probing} probing prediction \citep{langford2005estimating} as a tool to reduce
both components of the calibration loss at the same time, and also
\item  discuss in Section~\ref{se:curves} the role Brier curves \citep{hernandez2011brier} could play in controlling 
the grouping error if sufficiency cannot be achieved.
\end{itemize}


\section{Sufficiency}
\label{se:sufficiency}

In this section, we define a concept of sufficiency which is appropriate for the context of 
binary classification. At first glance, it differs from the more familiar statistical sufficiency with 
respect to a parametrised
family of distributions as defined for instance in \cite{Casella&Berger}. See 
Section~\ref{se:different} below for comments on the connections between `statistical
sufficiency' and sufficiency in the sense of the following Definition~\ref{de:sufficiency}. 

\subsection{Sufficiency and admissibility}
\label{se:admissible}

\cite{devroye1996probabilistic}, Definition~32.2, defined `sufficiency' in terms of random
variables $X$ and $Y$ and called a mapping $T$ with the image of $X$ as its domain a 
\emph{sufficient statistic} if for `for any set $A$, 
$P\{Y \in A\,|\,T(X), X\} = P\{Y\in A\,|\,T(X)\}$'. Their definition of sufficiency thus
is a special case of the following definition, with $\mathcal{M} = \sigma(Y)$, 
$\mathcal{H} = \sigma(X)$, and $\mathcal{G} = \sigma(T(X))$. 

\begin{definition}\label{de:sufficiency}
Let $(\Omega, \mathcal{A}, P)$ be a probability space and $\mathcal{M}$, $\mathcal{G}$ and $\mathcal{H}$ 
be sub-$\sigma$-fields of $\mathcal{A}$ such that $\mathcal{G} \subset \mathcal{H}$.
Then $\mathcal{G}$ is \emph{sufficient} for $\mathcal{H}$ 
with respect to $\mathcal{M}$ if for all $M\in \mathcal{M}$
\begin{equation}\label{eq:s.general}
 P[M\,|\,\mathcal{H}]\ = \ P[M\,|\,\mathcal{G}].
\end{equation} 
\end{definition}
\begin{remark}\label{rm:sufficiency}\emph{%
Under Assumption~\ref{as:setting}, let $\Psi = P[A\,|\,\mathcal{H}]$ as well as
\begin{equation}\label{eq:ex.sufficient}
\mathcal{G} \ = \ \sigma(\Psi) \quad \text{and}\quad 
    \mathcal{M} \ = \ \sigma(\{A\}) \ = \ \{\emptyset, \Omega, A, A^c\}.
 \end{equation}
Then it is easy to show that $\mathcal{G}$ is sufficient for $\mathcal{H}$  
with respect to $\mathcal{M}$. In applications, most of the time 
 $\sigma(\Psi) \subsetneqq \mathcal{H}$ will hold because $\Psi$ is one-dimensional 
 while $\mathcal{H}$ is likely to be generated by a multi-dimensional random vector. 
Hence it is indeed noteworthy that \eqref{eq:s.general} is true in this context.
}\hfill \qed
\end{remark}
As already observed by \cite{vanTrees}, page~29, Remark~\ref{rm:sufficiency} 
reflects the fact that $\Psi = P[A\,|\,\mathcal{H}]$
captures all the information contained in $\mathcal{H}$ that is relevant for inference regarding the 
event $A$. \cite{vanTrees} called this property sufficiency
without referring to any specific definition of the term. 

The question is debatable if the nesting condition $\mathcal{G} \subset \mathcal{H}$ should be part of a
definition of sufficiency as in Definition~\ref{de:sufficiency}. The author chose to include the
condition because inclusion comes at no cost for the discussion of calibration, 
the main topic of this paper. Nonetheless, definitions of sufficiency without nesting 
can make sense in other contexts. In Section~\ref{se:different} below, we also discuss the more general approach
to sufficiency by \cite{brocker2009reliability}.

At first glance, Definition~\ref{de:sufficiency} might not appear to be very helpful
when it comes to controlling the grouping loss as defined in Proposition~\ref{pr:brier}.
Indeed, under Assumption~\ref{as:setting}, Eq.~\eqref{eq:s.general} is true for $M = A$ if and only
if the grouping loss $E[(\Psi_{\mathcal{H}} - \Psi_{\mathcal{G}})^2]$ vanishes. 
Hence,
there is a condition for the grouping loss to be zero, but we have not yet got criteria
for the condition to apply. So far, by Remark~\ref{rm:sufficiency} we know that
a sufficient sub-$\sigma$-field of $\mathcal{H}$ exists. But, given that $\Psi = P[A\,|\,\mathcal{H}]$
is elusive, that observation does not really help.

At this point, it is useful to recall the notion of cost-sensitive learning and the related
extended notion of Bayes classifier -- see Section~32.3 of \cite{devroye1996probabilistic} 
for the details. Theorem~32.4 of \cite{devroye1996probabilistic}
suggests that study of the solutions to cost-sensitive classification problems could
provide further information for the problem of estimating posterior class probabilities.
For this purpose, we introduce the following definition.

\begin{definition}\label{de:Bayes}
Under Assumption~\ref{as:setting}, for $H\in \mathcal{H}$ and $t \in [0,1]$ the \emph{cost-weighted
mean loss} $L(H,t)$ is defined as
\begin{subequations}
\begin{equation}\label{eq:B}
L(H,t) \ =\ (1-t)\,P[A\cap H^c] + t\,P[A^c\cap H].
\end{equation}
The \emph{cost-weighted Bayes loss} $L_{\mathcal{H}}^\ast(t)$ is defined as
\begin{equation}\label{eq:Bstar}
\begin{split}
L_{\mathcal{H}}^\ast(t) & \ = \ \inf\limits_{H\in\mathcal{H}} L(H, t)\\
 & \ = \ L\bigl(\{P[A\,|\,\mathcal{H}]>t\}, t\bigr).
\end{split}
\end{equation}
\end{subequations}
\end{definition}

With the notation of Definition~\ref{de:Bayes}, sufficiency can be characterised as follows:
\begin{theorem}\label{th:devroye}
Under Assumption~\ref{as:setting}, let $\mathcal{G}$ denote a sub-$\sigma$-field
of $\mathcal{H}$ (hence $\mathcal{G}\subset\mathcal{H}$). With the cost-weighted
Bayes loss $L_{\mathcal{H}}^\ast(t)$  defined as in \eqref{eq:Bstar},
then $\mathcal{G}$ is sufficient for $\mathcal{H}$ with respect to $\sigma(\{A\})$ 
(or simply $A$) if and only if 
for all $t \in (0,1)$
\begin{equation}\label{eq:admissible}
L_{\mathcal{H}}^\ast(t)\ =\ L_{\mathcal{G}}^\ast(t).
\end{equation}
\end{theorem}
\begin{proof} The implication `sufficiency $\Rightarrow$ \eqref{eq:admissible}' is
obvious from \eqref{eq:Bstar}.

For the converse, assume that \eqref{eq:admissible} is true
for all $t \in (0,1)$. For $t \in (0,1)$ define $G(t) = \{P[A\,|\,\mathcal{G}] > t\}$.
Then as a consequence of \eqref{eq:admissible}, Theorem~\ref{th:Langford} below implies that
$$E\big[\bigl(Z - P[A\,|\,\mathcal{H}]\bigr)^2\bigr] \ 
    \le \ 2\,\int_0^1 L(G(t), t) - L_{\mathcal{H}}^\ast(t) \,dt\ = \ 0,$$
with $Z = \int_0^1 \mathbf{1}_{(0, P[A\,|\,\mathcal{G}])}(t)\,dt = P[A\,|\,\mathcal{G}]$.
This implies $P[A\,|\,\mathcal{H}] = P[A\,|\,\mathcal{G}]$. 
\end{proof}

In Example~\ref{ex:counter} of Section~\ref{se:different} below, we show that for the
implication `\eqref{eq:admissible} $\Rightarrow$ sufficiency' to be true, the nesting condition 
$\mathcal{G} \subset \mathcal{H}$ must be assumed to hold. Nesting of the $\sigma$-fields involved
comes naturally if Theorem~\ref{th:devroye} is phrased in terms of a transformation $T$ of the
features, as noted in the following corollary.
\begin{corollary}\label{co:admissible}
In the setting of Theorem~\ref{th:devroye}, assume there is a measurable mapping
$T:(\Omega, \mathcal{H}) \to (\Omega_T, \mathcal{A}_T)$ and let $\mathcal{G} = \sigma(T)$.
Then \eqref{eq:admissible} is equivalent to each of the following two properties:
\begin{itemize} 
\item[(i)] $\sigma(T)$ (or just $T$) is sufficient for $\mathcal{H}$ with respect to $A$, i.e.\ it
holds that $P[A\,|\,\mathcal{H}] = P[A\,|\,\sigma(T)]$.
\item[(ii)] There is a measurable function 
$G: (\Omega_T, \mathcal{A}_T) \to (\mathbb{R}, \mathcal{B}(\mathbb{R}))$
such $P[A\,|\,\mathcal{H}] = G(T)$.
\end{itemize}
\end{corollary}
\begin{proof}
 The equivalence of \eqref{eq:admissible} and (i) a special case of Theorem~\ref{th:devroye}.
The equivalence of (i) and (ii) is a direct consequence of the factorisation lemma. 
\end{proof}

In Corollary~\ref{co:admissible}, the statement `\eqref{eq:admissible} $\Leftrightarrow$ (i)' is
Theorem~32.6 of \cite{devroye1996probabilistic}. The statement `\eqref{eq:admissible} $\Leftrightarrow$ (ii)'
is Theorem~32.5 of \cite{devroye1996probabilistic}. \cite{devroye1996probabilistic}
called the mapping $T$ \emph{admissible} if \eqref{eq:admissible} holds with $\mathcal{G} = \sigma(T)$ since
such mappings can be applied to transform the features without changing the achievable minimum mean loss.

Further criteria for sufficiency -- in addition to Theorem~\ref{th:devroye} --
are desirable as \eqref{eq:admissible} might be hard if not impossible to show in practice.

\begin{proposition}\label{pr:increasing}
In the setting of Corollary~\ref{co:admissible}, assume that there are an $\mathcal{H}$-measurable 
real-valued random variable $T$ and a strictly increasing and continuous
function $F:(0,1)\to I \subset \mathbb{R}$ (where $I$ may denote a finite or infinite open interval)
such that for all $t\in(0,1)$ it holds that 
\begin{equation}\label{eq:exact}
L_{\mathcal{H}}^\ast(t)\ = \ L\bigl(\{T > F(t)\}, t),
\end{equation}
with $L_{\mathcal{H}}^\ast$ and $L$ as in Definition~\ref{de:Bayes}.
Denote by $G$ the inverse function of $F$ (hence $G$ is also strictly increasing and 
continuous).
Then it follows that $P[A\,|\,\mathcal{H}] = G(T)$, 
and $T$ is sufficient for $\mathcal{H}$ with respect to $A$. 
\end{proposition}
\begin{proof} Proposition~\ref{pr:increasing} follows from Theorem~\ref{th:Langford} in 
the same way as Theorem~\ref{th:devroye}. 
\end{proof}

Proposition~\ref{pr:increasing} could serve as justification for the familiar assumption that
a scoring classifier can be mapped by a strictly increasing transformation to the
posterior probability of the positive class \citep{zadrozny2001obtaining, kull2017betacalibration}.
However, this would require to have \eqref{eq:exact} satisfied to a reasonable degree.

A related criterion for sufficiency rather refers to the ranking of the instances provided 
by the scoring classifier than to optimality with respect to the mean cost-sensitive loss.

\begin{proposition}\label{pr:comon}
In the setting of Corollary~\ref{co:admissible}, assume that the mapping $T$ is real-valued.
Assume furthermore that $T$ and the probability $P[A\,|\,\mathcal{H}]$ of $A$ conditional on 
$\mathcal{H}$ are strongly comonotonic, i.e.\ it holds for all 
$\omega_1, \omega_2 \in \Omega$ that
\begin{equation*}
T(\omega_1) < T(\omega_2)  \quad \iff \quad P[A\,|\,\mathcal{H}](\omega_1) < 
 P[A\,|\,\mathcal{H}](\omega_2). 
\end{equation*}
Then $\sigma(T)$ (or just $T$) is sufficient for $\mathcal{H}$ with respect to $A$, i.e.\ it
holds that $P[A\,|\,\mathcal{H}] = P[A\,|\,\sigma(T)]$.
\end{proposition}
\begin{proof} Combining Proposition~4.5 of \cite{denneberg1994non} and
Proposition~2.3~(iii) of \cite{denneberg2006contribution}, we obtain the existence of
a strictly increasing function $\varphi:\mathbb{R}\to\mathbb{R}$ such that 
$P[A\,|\,\mathcal{H}] = \varphi(T)$. As monotonic functions are Borel measurable, 
sufficiency of $T$ follows by Corollary~\ref{co:admissible}. 
\end{proof}

By Proposition~\ref{pr:comon}, `Covariate Shift with Posterior Drift', a special
type of dataset shift introduced by \cite{Scott2019}, can alternatively be defined by
postulating that the posterior class probabilities on the source and target domains are 
strongly comonotonic. As a consequence, the posterior probability of the source domain
is sufficient with respect to the positive class labels also on the target domain. 

At first glance, Proposition~\ref{pr:comon} might appear encouraging with regard to the
chance to identify sufficient random variables $T$ because the comonotonicity criterion
does not look too demanding. For instance, \cite{Zadrozny&Elkan2002}, at the beginning of Section~1, 
expressed confidence that
this is typically achieved when learning classifiers: `Most supervised 
learning methods produce classifiers that output
scores $s(x)$ which can be used to rank the examples in the test set
from the most probable member to the least probable member of
a class $c$. That is, for two examples $x$ and $y$, if $s(x)<s(y)$ then
$P(c|x)<P(c|y)$.' 

However, while this statement is very plausible for
potential comonotonicity of a scoring classifier $T$ and posterior class 
probabilities $P[A\,|\,\sigma(T)]$, it is much less plausible with regard
to probabilities $P[A\,|\,\mathcal{H}]$ conditional on the full available information. 
Moreover, comonotonicity of two real-valued random variables $X_1, X_2$
implies that Kendall's $\tau$ and Spearman's rank correlation take the value one when
being applied to $X_1, X_2$ (see, for instance, \citealp{denuit2003simple}). 
This might not be easy to achieve, given the elusive nature of $P[A\,|\,\mathcal{H}]$.

Note also that under Assumption~\ref{as:setting} 
finding a mapping $T$ that is comonotonic with $P[A\,|\,\mathcal{H}]$
is equivalent to solving the bipartite ranking problem  
(\citealp{clemenccon2009tree}, Section~2.1).

\subsection{Alternative notions of sufficiency}
\label{se:different}

\cite{adragni2009sufficient} introduced in Definition~1.1 sufficiency in a setting
of random variables rather than $\sigma$-fields. Talking about 
`sufficient dimension reduction',
they considered univariate random variables $Y$ and $\mathbb{R}^p$-valued random vectors $X$ 
and defined `A reduction $R: \mathbb{R}^p \to \mathbb{R}^q$, $q \le p$, is sufficient 
if it satisfies one of
the following three statements:
\begin{itemize} 
\item[(i)] inverse reduction, $X\,|\,(Y,R(X))\ \sim\ X\,|\,R(X)$,
\item[(ii)] forward reduction, $Y\,|\,X\ \sim\ Y\,|\,R(X)$,
\item[(iii)] joint reduction $(X\perp~Y) \ |\ R(X)$,
\end{itemize}
where $\perp$ indicates independence, $\sim$ means identically distributed and $A | B$ refers
to the random vector $A$ given the vector $B$.'

A sufficient forward reduction in the sense of \cite{adragni2009sufficient} is a sufficient 
statistic in the sense of \cite{devroye1996probabilistic} and hence a special case of
Definition~\ref{de:sufficiency}. With regard to the relationship between `inverse reduction',
`forward reduction' and 'joint reduction', \cite{adragni2009sufficient} stated: `They are equivalent when $(Y,X)$
has a joint distribution.' In addition, according to Section~1(b) of \cite{adragni2009sufficient},
`If we consider a generic statistical problem and reinterpret $X$ as the total data
$D$ and $Y$ as the parameter $\theta$, then the condition for inverse reduction becomes
$D\,|\,(\theta,R)\ \sim\ D\,|\,R$ so that $R$ is a sufficient statistic. In this way, the definition of a
sufficient reduction encompasses \citeauthor{fisher1922mathematical}'s
(\citeyear{fisher1922mathematical}) classical definition of sufficiency.'
\cite{vanTrees} made on pages~34 and 35 a similar comment in the special context of binary
classification (see also \citealp{DeGroot&Fienberg1983}, Theorem~2).

The following proposition shows that also in the more general measure-theoretic setting of
this paper, there is a meaningful notion of `sufficient inverse reduction' that is equivalent
to Definition~\ref{de:sufficiency} of sufficiency if $\mathcal{G} \subset \mathcal{H}$. 
In Proposition~\ref{pr:forward}, the 
events $M$ of the $\sigma$-field $\mathcal{M}$ play the role the parameters $\theta$ play 
in the classical Fisher concept of sufficiency and the values of the variable $Y$ play in the definition
of inverse reduction of \cite{adragni2009sufficient}.
\begin{proposition}\label{pr:forward}
Let $(\Omega, \mathcal{A}, P)$ be a probability space and assume that 
$\mathcal{G}$, $\mathcal{H}$ and $\mathcal{M}$ are sub-$\sigma$-fields of $\mathcal{A}$.
With the notation $\mathcal{F}_1\vee\mathcal{F}_2 = \sigma(\mathcal{F}_1\cup\mathcal{F}_2)$
for $\sigma$-fields $\mathcal{F}_1$ and $\mathcal{F}_2$, 
then the following three statements are equivalent:
\begin{enumerate} 
\item[(i)] For all $H \in \mathcal{H}$, it holds that
$P[H\,|\,\mathcal{G}\vee\mathcal{M}] = P[H\,|\,\mathcal{G}]$.
\item[(ii)] For all $M\in\mathcal{M}$, it holds that 
$P[M \,|\, \mathcal{G}\vee\mathcal{H}] = P[M \,|\, \mathcal{G}]$.
\item[(iii)] For all $H \in \mathcal{H}$ and all $M\in\mathcal{M}$, it holds that
$$P[H\cap M\,|\,\mathcal{G}] = P[H\,|\,\mathcal{G}]\,P[M\,|\,\mathcal{G}].$$
\end{enumerate}
\end{proposition}
The proof of Proposition~\ref{pr:forward} is straightforward, by making use of the properties of
conditional probabilities.
 To reconcile Proposition~\ref{pr:forward} with the notion of sufficient dimension
reduction used by \cite{adragni2009sufficient}, choose
\begin{equation}
\mathcal{M} = \sigma(Y), \qquad \mathcal{H} = \sigma(X), \qquad
\mathcal{G} = \sigma(R(X)) \subset \mathcal{H}.
\end{equation}
Then Proposition~\ref{pr:forward}~(i) describes sufficient inverse reduction, (ii)
generalises sufficient forward reduction by not requiring $\mathcal{G}\subset \mathcal{H}$,
and (iii) is another way to express sufficient joint reduction.

For a better understanding of Proposition~\ref{pr:forward}, it is helpful to
think of $\mathcal{M}$ as a collection of events to be predicted and of $\mathcal{H}$
as the maximal collection of observable events the prediction can be based on. Regarding
$\mathcal{G}$, the following three cases may be considered separately:
\begin{description} 
 \item[Case 1:] $\mathcal{G}\subset \mathcal{H}$. Then
Proposition~\ref{pr:forward}~(ii) is equivalent to Definition~\ref{de:sufficiency},
and we are in the situation of Section~\ref{se:admissible} above 
where aggregation of observations without loss 
of prediction quality is sought 
in order to facilitate computations.
\item[Case 2:] $\mathcal{G}\supset \mathcal{H}$. Then Proposition~\ref{pr:forward} does not
provide much value because all three statements are obviously true as $\mathcal{H}$ basically
becomes redundant in the statements of the proposition. However, due to our interpretation
of $\mathcal{H}$ as maximal collection of observable events, in this case $\mathcal{G}$ may contain
unobservable events. This would make it unfit to substitute for $\mathcal{H}$.
But instead $\mathcal{G}$ could replace $\mathcal{M}$ as regression target, thanks 
to the tower property:
\begin{equation}\label{eq:broecker}
P[M\,|\,\mathcal{H}] \ = \ E\bigl[P[M\,|\,\mathcal{G}]\bigm|\mathcal{H}\bigr].
\end{equation}
So-called shadow ratings are an example for this approach 
\citep{erlenmaier2011shadow}. Another example is the situation where weather forecasters
$\mathcal{G}$  and $\mathcal{H}$ have access to a common set of observation stations but 
$\mathcal{G}$  gets extra data from some additional stations not in the common set.
\item[Case 3:] $\mathcal{G}\not\subset \mathcal{H}$ and $\mathcal{H}\not\subset \mathcal{G}$.
Proposition~\ref{pr:forward} is non-trivial in this case. It is similar to case~2 in so far 
as again $\mathcal{G}$ may contain unobservable events. However, \eqref{eq:broecker} is
not automatically true in this situation either but is implied by any of the equivalent 
statements of Proposition~\ref{pr:forward}.
 \end{description} 
 
An example for case~3 could be a situation with two weather forecasters. They might
rely on the same observation stations but deploy different meter-readers who make
independent reading errors. Formally, this situation could be described as in
the following example.

\begin{example}\emph{%
Under Assumption~\ref{as:setting}, let $\mathcal{G}$ be a sub-$\sigma$-field of $\mathcal{H}$
that is sufficient for $\mathcal{H}$ with respect to $A$ in the sense of Definition~\ref{de:sufficiency}.
Assume that $\mathcal{I}$ and $\mathcal{J}$ are further sub-$\sigma$-fields of $\mathcal{A}$ such
that $\mathcal{I}\vee\mathcal{J}$ and $\mathcal{H}\vee \{\emptyset, \Omega, A, A^c\}$ are 
independent. Let $\mathcal{G}^\ast = \mathcal{G} \vee \mathcal{I}$ and
$\mathcal{H}^\ast = \mathcal{H} \vee \mathcal{J}$. Then it follows that
$$P[A\,|\,\mathcal{H}^\ast \vee \mathcal{G}^\ast] \ = \
P[A\,|\,\mathcal{G}^\ast],$$
such that for $\mathcal{G}^\ast$, $\mathcal{H}^\ast$ and $\mathcal{M} = \{\emptyset, \Omega, A, A^c\}$
statement~(ii) of Proposition~\ref{pr:forward} is true.
}
\hfill \qed
\end{example}

Is the condition $\mathcal{G} \subset \mathcal{H}$ actually not needed in Definition~\ref{de:sufficiency}? 
Would Theorem~\ref{th:devroye}
still be true if Definition~\ref{de:sufficiency} were replaced by statement~(ii) of 
Proposition~\ref{pr:forward} without requiring $\mathcal{G}\subset \mathcal{H}$?

Indeed,  \cite{schervish1989general} (in Definition~10.1) and \cite{brocker2009reliability}
(in Eq.~(14))
defined sufficiency by means of \eqref{eq:broecker}, without requiring 
$\mathcal{G}\subset \mathcal{H}$ or $\mathcal{H}\subset \mathcal{G}$.
\cite{brocker2009reliability} provided in Section~4 another non-trivial example 
for the third case we described above.

Theorem~10.2 of \cite{schervish1989general} seems to suggest that Theorem~\ref{th:devroye}
is true under \eqref{eq:broecker} instead of sufficiency according to Definition~\ref{de:sufficiency}. 
Below we show that Theorem~10.2 of \cite{schervish1989general} needs to be carefully interpreted.
\begin{itemize} 
\item In Proposition~\ref{pr:broecker}, we prove that indeed \eqref{eq:broecker} 
implies $L^\ast_{\mathcal{G}}(t) \le L^\ast_{\mathcal{H}}(t)$ for the cost-weighted Bayes losses
associated with $\mathcal{G}$ and $\mathcal{H}$ respectively in the sense of Definition~\ref{de:Bayes}.
\item However, in Example~\ref{ex:counter} we show that 
$L^\ast_{\mathcal{G}}(t) \le L^\ast_{\mathcal{H}}(t)$ for all $t \in (0,1)$ does not always
imply \eqref{eq:broecker} (and therefore neither of
the three statements of Proposition~\ref{pr:forward}) if $\mathcal{G} \subset \mathcal{H}$ does
not hold.
\end{itemize}
 
\begin{proposition}\label{pr:broecker}
Under Assumption~\ref{as:setting}, let $\mathcal{G}$ be another sub-$\sigma$-field of $\mathcal{A}$
and assume that \eqref{eq:broecker} is true with $M = A$. Then for all
$t \in (0,1)$ it holds that
\begin{equation}\label{eq:schervish}
L^\ast_{\mathcal{G}}(t) \le L^\ast_{\mathcal{H}}(t).
\end{equation}
\end{proposition}
\begin{proof} Fix any $H \in \mathcal{H}$ and $0 < t < 1$. Then we obtain
\begin{align}
L(H, t) & \ = \ (1-t)\,P[A\cap H^c] + t\,P[A^c \cap H] \notag\\
    & \ = \ \ (1-t)\,E\bigl[\mathbf{1}_{H^c}\,P[A\,|\,\mathcal{H}]\bigr]  + 
        t\,E\bigl[\mathbf{1}_{H}\,(1-P[A\,|\,\mathcal{H}])\bigr]\notag\\
   & \ = \ (1-t)\,E\bigl[\mathbf{1}_{H^c}\,E\bigl[P[A\,|\,\mathcal{G}]\bigm|\mathcal{H}\bigr]\bigr]  + 
        t\,E\bigl[\mathbf{1}_{H}\,\bigl(1-E\bigl[P[A\,|\,\mathcal{G}]\bigm|\mathcal{H}\bigr]\bigr)\bigr] \notag\\
   & \  = \ 
   (1-t)\,E\bigl[(1-P[H\,|\,\mathcal{G}])\,P[A\,|\,\mathcal{G}]\bigr]  + 
        t\,E\bigl[P[H\,|\,\mathcal{G}]\,(1-P[A\,|\,\mathcal{G}])\bigr]\notag\\
   & \ = \ (1-t)\,E\bigl[\mathbf{1}_A\,(1-P[H\,|\,\mathcal{G}])\bigr]  + 
        t\,E\bigl[\mathbf{1}_{A^c}\,P[H\,|\,\mathcal{G}]\bigr]. \label{eq:imply}     
\end{align}
Closer inspection of Theorem~32.4 of \cite{devroye1996probabilistic} shows that
it does not only apply to `decision functions' but also to
`randomised decision classifiers' in the sense of
\cite{tasche2018plug}.
Therefore, \eqref{eq:imply}  implies 
$L(H, t) \ge L^\ast_{\mathcal{G}}(t)$ and hence also \eqref{eq:schervish}.
\end{proof}

Proposition~\ref{pr:broecker} in conjunction with \eqref{eq:area*} from 
Section~\ref{se:prelim} below implies
that under \eqref{eq:broecker} the refinement loss (in terms of the Brier score) 
of $\mathcal{G}$ is not greater than the refinement loss of 
$\mathcal{H}$ -- which was proven by \cite{brocker2009reliability} for resolution and refinement
defined in terms of general proper scores.

In order to show that \eqref{eq:schervish} for all $t \in (0,1)$ does in general not imply 
\eqref{eq:broecker}, we  make use
of the following lemma.

\begin{lemma}\label{le:counter}
Under Assumption~\ref{as:setting}, let $\mathcal{G}$ be another sub-$\sigma$-field of $\mathcal{A}$
such that $\mathcal{H}$ and $\mathcal{G}$ are independent conditional on $A$ and independent conditional
on $A^c$. Assume $L^\ast_{\mathcal{G}}(t) > 0$ for some $t\in (0,1)$. 
Then \eqref{eq:broecker} (with $M = A$) implies that $A$ and $\mathcal{H}$ are independent
(hence $P[A\,|\,\mathcal{H}] = P[A]$ is constant).
\end{lemma}
\begin{proof} For any $H\in \mathcal{H}$ we obtain
\begin{align*}
P[A \cap H] & \ = \ E\bigl[\mathbf{1}_H\,P[A\,|\,\mathcal{H}]\bigr]\\
    & \ = \ E\bigl[\mathbf{1}_H\,E\bigl[P[A\,|\,\mathcal{G}]\,\big|\,\mathcal{H}\bigr]\bigr]\\
    & \ = \ E\bigl[\mathbf{1}_H\,P[A\,|\,\mathcal{G}]\bigr]\\
    & \ = \ P[A]\,E\bigl[\mathbf{1}_H\,P[A\,|\,\mathcal{G}]\,\big|\,A\bigr] 
   +  P[A^c]\,E\bigl[\mathbf{1}_H\,P[A\,|\,\mathcal{G}]\,\big|\,A^c\bigr]\\
    & \ = \ P[A]\,P[H\,|\,A]\,E\bigl[P[A\,|\,\mathcal{G}]\,\big|\,A\bigr]
   +  P[A^c]\,P[H\,|\,A^c]\,E\bigl[P[A\,|\,\mathcal{G}]\,\big|\,A^c\bigr].
\end{align*}
This implies for all $H \in \mathcal{H}$
\begin{subequations}
\begin{equation}\label{eq:counterH}
P[A]\,P[H\,|\,A]\,\Big(1-E\bigl[P[A\,|\,\mathcal{G}]\,\big|\,A\bigr]\Big)
   \ = \  P[A^c]\,P[H\,|\,A^c]\,E\bigl[P[A\,|\,\mathcal{G}]\,\big|\,A^c\bigr].
\end{equation}
Choosing $H = \Omega$ gives
\begin{equation}\label{eq:counter}
P[A]\,\Big(1-E\bigl[P[A\,|\,\mathcal{G}]\,\big|\,A\bigr]\Big)
   \ = \  P[A^c]\,E\bigl[P[A\,|\,\mathcal{G}]\,\big|\,A^c\bigr].
\end{equation}
\end{subequations}
$E\bigl[P[A\,|\,\mathcal{G}]\,\big|\,A^c\bigr] = 0$ would imply 
$E\bigl[P[A\,|\,\mathcal{G}]\,\big|\,A\bigr] = 1$ (also vice versa) and hence
$P[A\,|\,\mathcal{G}] = 0$  on $A^c$ and $P[A\,|\,\mathcal{G}] = 1$ on $A$, almost surely.
From this it would follow that
$$L^\ast_{\mathcal{G}}(t) \ = \
(1-t)\,P\bigl[A \cap\{P[A\,|\,\mathcal{G}]\le t\}\bigr] + 
t\,P\bigl[A^c \cap\{P[A\,|\,\mathcal{G}]> t\}\bigr] \ = \ 0,$$
contradicting the assumption $L^\ast_{\mathcal{G}}(t) > 0$.
Hence \eqref{eq:counterH} and \eqref{eq:counter} imply
$$P[H\,|\,A] \ = \ P[H\,|\,A^c]\qquad \text{for all}\ H \in \mathcal{H}.$$
In plain language, this means that $A$ and $\mathcal{H}$ are independent. 
\end{proof}
\begin{example}\label{ex:counter}\emph{%
Figure~\ref{fig:Brier2} shows curves $t \mapsto L^\ast_{\mathcal{H}}(t)$ and  
$t \mapsto L^\ast_{\mathcal{G}}(t)$ for $0 \le t \le 1$ for different choices of
$\mathcal{H}$ and $\mathcal{G}$, with the cost-weighted Bayes loss $L^\ast_{\mathcal{H}}(t)$
defined by \eqref{eq:Bstar}.
\begin{itemize} 
\item For Figure~\ref{fig:Brier2}, under Assumption~\ref{as:setting} a model based on bivariate 
normal feature distributions with different mean
vectors but identical covariance matrices has been chosen. 
The correlation between
the components is assumed to be zero.
Therefore the two components $X_1$ and $X_2$ of the normal distributions are independent 
conditional on the classes (i.e.\ given $A$ and $A^c$ respectively in the setting of
Assumption~\ref{as:setting}). We define $\mathcal{H} = \sigma(X_1)$ and $\mathcal{G} = \sigma(X_2)$
such that $\mathcal{G}\not\subset \mathcal{H}$ and $\mathcal{H}\not\subset \mathcal{G}$.
\item The dashed curve of Figure~\ref{fig:Brier2} shows the curve 
$t \mapsto L^\ast_{\mathcal{H}\vee\mathcal{G}}(t)$ of the 
true posterior class probability, making use of
the full information available (i.e.\ both marginal components). Note that the maximum of this
optimal curve is lower than the maximum of the optimal curve in Figure~\ref{fig:Brier} below. This
is due to the fact that for this example zero correlation between the components of class conditional
distribution is assumed.
\item The dotted curve shows $t \mapsto L^\ast_{\mathcal{H}}(t)$ (i.e.\ the curve for $X_1$, 
hence based on partial information only).
\item The dash-dotted curve shows $t \mapsto L^\ast_{\mathcal{G}}(t)$ (i.e.\ the curve for $X_2$, 
also based on partial information only).
\item The solid curve shows $t \mapsto L^\ast_{\{\emptyset, \Omega\}}(t) = 
\min\bigl((1-t)\,P[A], t\,(1-P[A])\bigr)$, the curve for the
constant `prior estimate' $P[A]$. 
\end{itemize}
Clearly, the curve for $X_2$ is dominated by the curve for $X_1$, i.e.\ \eqref{eq:schervish} 
is true for all $t \in (0,1)$. However, if $\mathcal{G} = \sigma(X_2)$ and 
$\mathcal{H} = \sigma(X_1)$ satisfied \eqref{eq:broecker}, by Lemma~\ref{le:counter} $X_1$ and $A$ would
be independent. Hence we would have $P[A\,|\,\sigma(X_1)] = P[A]$ constant and the curve for $X_1$
would be the solid curve. This is clearly not the case since the solid and the dotted curves do
not coincide. It follows that \eqref{eq:broecker} cannot be true for $\mathcal{G} = \sigma(X_2)$ and 
$\mathcal{H} = \sigma(X_1)$.
}
 \hfill \qed
\end{example}
\begin{figure}
\begin{center}
\includegraphics[width=10cm]{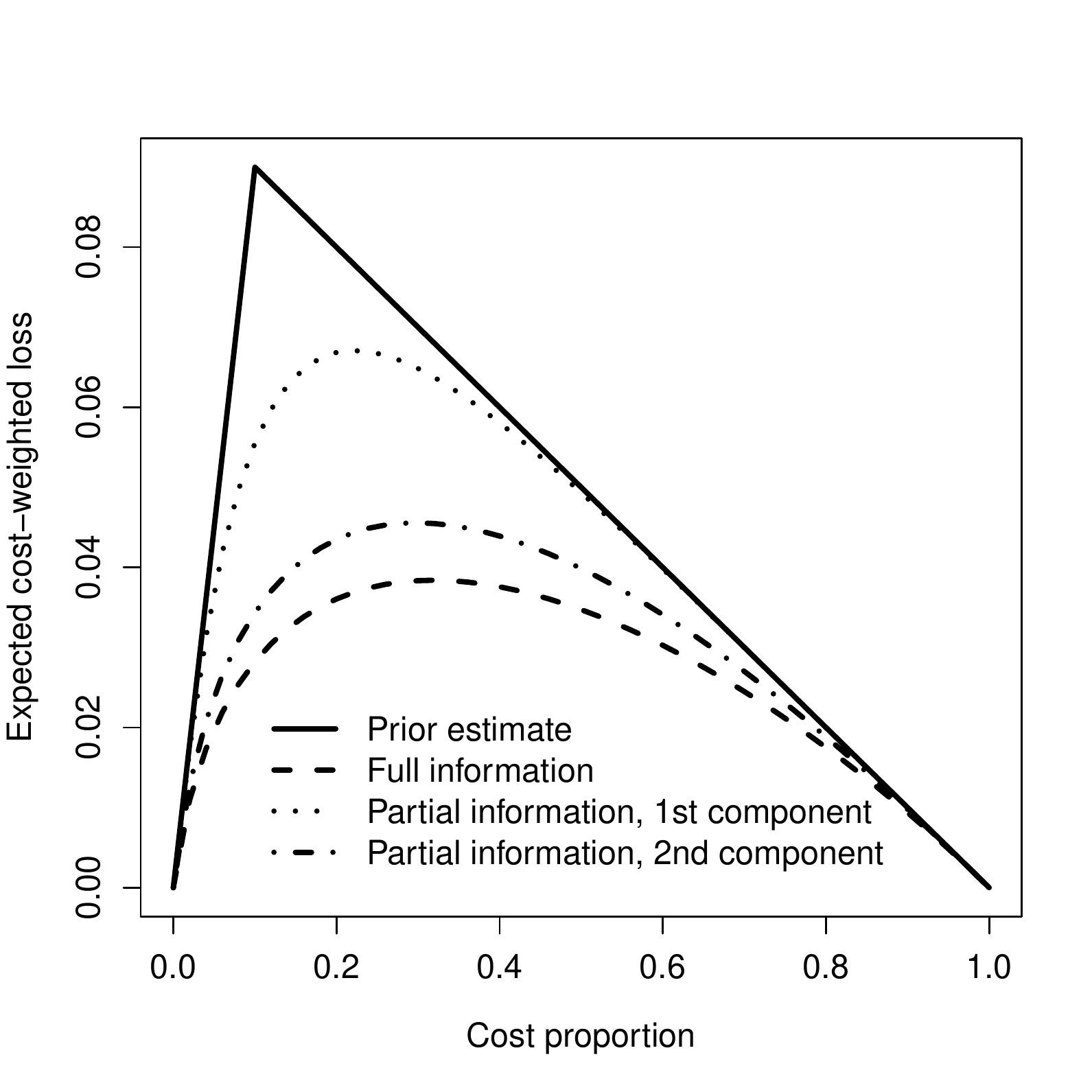}
\caption{{\small{}Illustration of Example~\ref{ex:counter}. The conditional feature distributions are
uncorrelated bivariate normal with identical covariance matrices. The prior probability of the positive class is
10\%.}\label{fig:Brier2}}
\end{center}
\end{figure}
How to reconcile Example~\ref{ex:counter} with Theorem~10.2 of \cite{schervish1989general}?
Unfortunately, \cite{schervish1989general} did not provide a detailed proof of the theorem but
only suggested the building stones needed for a proof. Therefore it seems possible that
in some of the steps of the proof an assumption would have to be made that forecaster A had no more 
information at her disposal than forecaster B (in terms of this paper: $\mathcal{G}
\subset \mathcal{H}$). This way, Theorem~10.2 of \cite{schervish1989general} would contradict
neither Example~\ref{ex:counter} nor Theorem~\ref{th:devroye}.

As observed in \cite[p.~974]{Krueger&Ziegel2021}, according to Strassen's theorem \citep{Strassen1965},
inequality \eqref{eq:schervish} holds if and only if there are random variables $Z_{\mathcal{G}}$ and
$Z_{\mathcal{H}}$ on a possibly different probability space such that 
\begin{itemize}
\item the distributions of $Z_{\mathcal{G}}$ and $P[A\,|\,\mathcal{G}]$ are equal,
\item the distributions of $Z_{\mathcal{H}}$ and $P[A\,|\,\mathcal{H}]$ are equal, and
\item it holds that $E\bigl[Z_{\mathcal{G}}\,|\,\sigma(Z_{\mathcal{H}})\bigr] =
    Z_{\mathcal{H}}$.
\end{itemize}
This observation does not contradict Example~\ref{ex:counter} because there we have shown 
that \eqref{eq:broecker} does not hold almost surely (i.e.\ with probability 1). This does not
exclude the possibility that equality holds in distribution.

\section{Probing reduction}
\label{se:probing}

What if we do not find a sufficient scoring classifier $T$ for $\mathcal{H}$, for instance when
we cannot be sure to have reached the minimum on the left-hand side of \eqref{eq:admissible}, 
or when there are doubts that $T$ is comonotonic with $P[A\,|\,\mathcal{H}]$?
By 
Corollary~\ref{co:admissible}, without sufficiency of $T$ the relation 
$P[A\,|\,\mathcal{H}] = G(T)$ cannot be achieved. Is then everything lost if
no sufficient scoring classifier can be identified?

The criteria for sufficiency we have presented in Section~\ref{se:admissible} have
in common that they involve finding a single variable $T$ which solves an
infinite number of optimisation problems at the same time, see for instance 
Proposition~\ref{pr:increasing}.
This might be a difficult task.
It would be nice to be able to separately solve the optimisation problems, resulting
in  a family of optimal classifiers, and then to combine the classifiers to 
one sufficient scoring classifier $T$ which provides the required simultaneous solution
for all the optimisation problems.

Without referring to the notions of sufficiency or bipartite ranking, 
\cite{langford2005estimating} presented `probing reduction' as a way to
achieve this feat. They described their method with the statement
`reduce learning an estimator of class probability membership [sic] to learning
binary classifiers'. As part of their theoretical analysis, \cite{langford2005estimating}
proved a theorem on `probing error transformation' which implies 
the results of \cite{devroye1996probabilistic}, Section~32.3, on sufficiency and
admissibility and provides a constructive way to estimate posterior probabilities
exploiting all the information available.

In the following, we present a slightly more abstract version of the `probing predictor' and 
Theorem~2 of \cite{langford2005estimating}.

\subsection{Preliminaries}
\label{se:prelim}
Under Assumption~\ref{as:setting}, let $Z$ be a real-valued $\mathcal{A}$-measurable random
variable. For $t\in[0,1]$ and $\omega\in\Omega$ let
\begin{subequations}
\begin{equation}\label{eq:basic}
\ell_Z(\omega, t) \ = \ (1-t)\,\mathbf{1}_A(\omega)\,\mathbf{1}_{[Z(\omega),1]}(t)
  + t\,\mathbf{1}_{A^c}(\omega)\,\mathbf{1}_{[0,Z(\omega))}(t).
\end{equation}
For fixed $\omega\in\Omega$, then $t \mapsto \ell_Z(\omega, t)$ is right-continuous on
$[0,1)$ and has left limits for $t\in (0,1]$. 
As a consequence $\ell_Z(\omega, t)$ is jointly measurable in $(\omega,t)$, see 
any standard textbook on the theory of stochastic processes. Moreover, with $L(H,t)$
and $L_{\mathcal{H}}^\ast(t)$ defined as in Definition~\ref{de:Bayes}, it holds that
\begin{equation}\label{eq:ell}
\begin{split}
L(\{Z>t\}, t) & \ = \ E[\ell_Z(\cdot, t)] \qquad\text{and}\\
L_{\mathcal{H}}^\ast(t) & \ = \ E[\ell_\Psi(\cdot, t)],
\end{split}
\end{equation}
\end{subequations}
for $\Psi = P[A\,|\,\mathcal{H}]$. The double integrals appearing in the following
proposition therefore are well-defined.
The result was shown by \cite{hernandez2012unified}, Theorem~14 and Theorem~29. 
It had been proven before
by \cite{langford2005estimating} as part of their proof of Theorem~2. Thus the result is needed
for the proof of Theorem~\ref{th:Langford} below but is also of interest of its own.

\begin{proposition}\label{pr:area}
Under Assumption~\ref{as:setting}, with the cost-weighted mean loss defined as
in \eqref{eq:B} and the Brier score $BS$ as in Definition~\ref{de:brier}, 
for any $\mathcal{A}$-measurable random variable $Z$ with values in $[0,1]$ it holds that
\begin{subequations}
\begin{equation}\label{eq:area}
2 \int_0^1 L(\{Z > t\},\,t)\,dt \ = \ BS(Z).
\end{equation}
With the cost-weighted Bayes loss defined as in \eqref{eq:Bstar} and $\Psi = P[A\,|\,\mathcal{H}]$, 
the refinement loss $E[\Psi\,(1-\Psi)]$ (see Proposition~\ref{pr:brier})
can be represented as 
\begin{equation}\label{eq:area*}
2 \int_0^1 L_{\mathcal{H}}^\ast(t)\,dt \ = \ E\bigl[\Psi\,(1-\Psi)\bigr].
\end{equation}
\end{subequations}
\end{proposition} 
Note that \eqref{eq:area} and \eqref{eq:area*} differ from the statements 
in Theorem~14 and Theorem~29 respectively of \cite{hernandez2012unified} by
the factor $1/2$ being applied to the Brier score on the right-hand side. This is a consequence
of \cite{hernandez2012unified} having scaled the Bayes error terms by the factor 2.\\[1ex]
\noindent\textbf{Proof of Proposition~\ref{pr:area}.} 
Starting from the left-hand side of \eqref{eq:area}, we obtain by making use
of \eqref{eq:ell} and Fubini's theorem
\begin{align*}
\int_0^1 L(\{Z > t\},\,t)\,dt & \ = \ \int_0^1 E[\ell_Z(\cdot, t)]\,dt\\
& \ = \ E\left[ \mathbf{1}_A \int_Z^1 (1-t)\,dt + \mathbf{1}_{A^c} \int_0^Z t\,dt \right]\\
& \ = \ E\bigl[\mathbf{1}_A\,(1-Z)^2/2 + \mathbf{1}_{A^c}\,Z^2/2\bigr].
\end{align*}
Observing that $\mathbf{1}_A\,(1-Z)^2 + \mathbf{1}_{A^c}\,Z^2 = (\mathbf{1}_A - Z)^2$ implies 
\eqref{eq:area}. 
Eq.~\eqref{eq:area*} follows by combining \eqref{eq:area} and \eqref{eq:decomp} (with $Z = \Psi$).
\hfill $\Box$

\subsection{Main result} 
The following theorem basically is a restatement of 
Theorem~2 of \cite{langford2005estimating}, with some changes to the notation and
a different representation of the `probing predictor' in \eqref{eq:combination}. See
also Remark~\ref{rm:Langford} below for more details on how Theorem~\ref{th:Langford} differs
from the original version given by \cite{langford2005estimating}.

\begin{theorem} \label{th:Langford}
Under Assumption~\ref{as:setting}, define the cost-weighted mean loss $L(H,t)$ 
and the cost-weighted Bayes loss $L_{\mathcal{H}}^\ast(t)$
for $H \in \mathcal{H}$ and $0 \le t \le 1$ as in Definition~\ref{de:Bayes}.
Assume that for each $t\in (0,1)$ a classifier represented by $H(t) \in \mathcal{H}$ is given such that the
function $h: (0,1) \times \Omega \to \{0,1\}$ with 
$$h(t, \omega) \ = \ \begin{cases}
1, & \text{if}\ \omega\in H(t),\\
0, & \text{if}\ \omega\notin H(t),
\end{cases}$$
is $\mathcal{B}((0,1)) \otimes \mathcal{H}$-measurable. Define the $\mathcal{H}$-measurable 
random variable $Z$ by 
\begin{subequations}
\begin{equation}\label{eq:combination}
Z(\omega) \ =\ \int_0^1 h(t, \omega)\,dt.
\end{equation}
Let $\Psi = P[A\,|\,\mathcal{H}]$. Then it follows for the calibration loss 
$E\bigl[(Z - \Psi)^2\bigr]$ of $Z$ that
\begin{equation}\label{eq:Langford}
\begin{split}
E\bigl[(Z - \Psi)^2\bigr] & \ = \ 2\,\int_0^1 L(\{Z > t\}, t) - L_{\mathcal{H}}^\ast(t) \,dt\\
 & \ \le \ 2\,\int_0^1 L(H(t), t) - L_{\mathcal{H}}^\ast(t) \,dt
\end{split} 
\end{equation}
\end{subequations}
\end{theorem}
See Appendix~\ref{se:Langford} for a proof of Theorem~\ref{th:Langford}.

\begin{remark}\label{rm:Langford}\emph{%
Comments on Theorem~\ref{th:Langford}:
\begin{enumerate} 
\item Actually, \cite{langford2005estimating} proved in Theorem~2 the 
following version of Theorem~\ref{th:Langford} (rephrased in the notation of this paper):\\
\emph{With the notation $\bar{L}(H,t) = \frac{1-t}{t}\,P[A\cap H^c] + P[A^c\cap H]$  and 
$N_t = \frac{1-t}{t}\,P[A] + P[A^c]$ for $t \in(0,1)$, it holds that}
\begin{equation}
E\bigl[(Z - \Psi)^2\bigr] \ \le \ 2\,\max(P[A], 1-P[A])\, 
    \int_0^1 \frac{\bar{L}(H(t),t) - \bar{L}(\{\Psi>t\},t)}{N_t}\,dt.
\end{equation}
Observing that $t\,N_t \le \max(P[A], 1-P[A])$, one finds that
\begin{eqnarray*}
\lefteqn{\max(P[A], 1-P[A])\, 
    \int_0^1 \frac{\bar{L}(H(t),t) - \bar{L}(\{\Psi>t\},t)}{N_t}\,dt}\\
 &    \ge \ & \int_0^1 t\,\bigl(\bar{L}(H(t),t) - \bar{L}(\{\Psi>t\},t)\bigr)\,dt \\
  &  = \  &\int_0^1 L(H(t), t) - L_{\mathcal{H}}^\ast(t)  \,dt.   
\end{eqnarray*}
At first glance, therefore, the upper bound for $E\bigl[(Z - \Psi)^2\bigr]$ of Theorem~2 of
\cite{langford2005estimating} is less strict than the upper bound provided by Theorem~\ref{th:Langford}.
However, a closer inspection of the proof of Theorem~2 of \cite{langford2005estimating} reveals
that they indeed proved \eqref{eq:Langford} but due
to their different notion of approximation error ended up in stating an apparently
weaker result. With a view to applications, however, $\bar{L}(H,t)$ reflects
the fact that the weight-dependent classifiers from the right-hand side of \eqref{eq:combination}
likely would be trained on samples in which only one class is re-weighted -- not 
on samples with both classes re-weighted as suggested by definition \eqref{eq:B}
of $L(H, t)$.
\item The `probing predictor' $\hat{p}(x)$ of \cite{langford2005estimating} 
(which corresponds to $Z$ in the notation of 
this paper) is obtained by `sort[ing] the results of the classifiers to make the sequence monotonic
(all zeroes before all ones)'. Defining $Z$ as the average of the $h(t, \cdot)$ has the same 
effect as sorting but has the advantage to be viable also in the infinite sample case of this paper.
\item \cite{langford2005estimating}  commented on their Theorem~2 as follows: `This is a strong
theorem in the sense that it ties the error in the probability predictions to the \emph{average relative}
importance weighted loss of the classifiers. Using the average loss over $w$ [$t$ in 
the notation of this paper] results in a more powerful
statement than using the maximal loss, because it is easier to obtain a set of classifiers which have
small average loss than to obtain a set of classifiers, all with a small loss. Using a loss that is relative
to the loss of the Bayes optimal classifier means that the theorem applies even when the fundamental noise rate
is large.' 
\item In addition to the comment of \cite{langford2005estimating} on the theorem as quoted in 3), 
it might be worthwhile 
to point out that the `probing predictor' $\hat{p}(x)$ (or $Z$ in the notation of this paper) performs better
on average than the classifiers $H_t$, $0 <t < 1$ that are the input to the estimation procedure (besides 
the data). This is clearly stated in Theorem~\ref{th:Langford} above but was implicitly stated
already in the proof of Theorem~2 of \cite{langford2005estimating}.
\end{enumerate}
}
\end{remark}

The following example illustrates how \eqref{eq:combination} can be used 
to combine two estimators $Z_1$ and $Z_2$ of the posterior class probability $P[A\,|\,\mathcal{H}]$ 
to form another estimator $Z$ with a potentially smaller Brier score.

\begin{example}\emph{%
Let $Z_1, Z_2$ be $\mathcal{H}$-measurable random variables with values in $[0,1]$ and
let $0 < z < 1$ be fixed. Define
\begin{align*}
Z & \ = \ \int_0^z \mathbf{1}_{\{Z_1>t\}}\,dt + \int_z^1 \mathbf{1}_{\{Z_2>t\}}\,dt\\
    & \ = \ \min(z, Z_1) + (Z_2-z)\,\mathbf{1}_{\{Z_2>z\}}.
\end{align*}
If $z$ is chosen in such a way that 
$$L(\{Z_1 > t\}, t) \le L(\{Z_2 > t\}, t) \quad \iff \quad t \le z,$$ 
then by \eqref{eq:Langford} it follows that
$$BS(Z) \ \le \ \min\bigl(BS(Z_1), BS(Z_2)\bigr).$$
Such a number $z$ could be identified by inspection of the Brier curves of $Z_1$ and $Z_2$,
see for example the criss-crossing dashed and dash-dotted curves in 
Figure~\ref{fig:Brier} below. 
} \hfill \qed
\end{example}

\section{Brier curves}
\label{se:curves}

The name and notion of Brier curves as discussed in this section were introduced by \cite{hernandez2011brier}.
At about the same time, the concept was considered by
\cite{reid2011information} who called it `risk curves for costs'. However, such 
curves -- without being given specific names -- were presented much earlier in the literature
(see \citealp{Ehm&Gneiting&Jordan&Krueger}, p.~519, for more details).

\subsection{Definition and properties}

\begin{definition}[Brier curve]\label{de:curves}
For a probabilistic classifier $Z$ (i.e.\ an $\mathcal{H}$-measurable $[0,1]$-valued random variable), 
the \emph{Brier curve} $t \mapsto B(t)$ is defined as 
$$B(t) \ = \ L(\{Z > t\}, t), \qquad t \in [0,1],$$
where $L(\cdot, t)$ denotes the cost-weighted mean loss of Definition~\ref{de:Bayes}.
\end{definition}
The definition of Brier curve used here differs from the definition in Eq.~(10) of \cite{hernandez2011brier}
by a factor~2. By multiplying with 2, one would achieve equality between the area under the Brier curve
and the Brier score (see comments below).
The `Brier curve for skew' as defined in Eq.~(11) of \cite{hernandez2011brier} is
just the Brier curve as in Definition~\ref{de:curves} (but scaled with factor 2) with prior probability
$P[A] = 1/2$. Both \cite{reid2011information} and \cite{hernandez2011brier} presented
charts showing idealised population-level and sample-based, empirical Brier curves. See also
Figure~\ref{fig:Brier} below for examples of Brier curves.

The following properties hold for Brier curves in general:
\begin{itemize} 
\item In terms of the two conditional distribution functions $t \mapsto P[Z\le t\,|\,A]$
and $t \mapsto P[Z\le t\,|\,A^c]$, the Brier curve for $Z$ can be represented as
\begin{equation}\label{eq:distrib}
B(t) \ = \ (1-t)\,P[A]\,P[Z\le t\,|\,A] + t\,(1-P[A])\,(1-P[Z\le t\,|\,A^c]), \  
0 \le t \le 1.
\end{equation}
\item By \eqref{eq:distrib}, $t \mapsto B(t)$ is right-continuous on
$[0,1)$ and has left limits for all $t \in (0,1]$.
\item $B(t) \ge 0$ for all $t \in [0,1]$, $B(1) = 0$, $B(0) = P[A]\,P[Z=0\,|\,A]$.
\item Denoting by $B_Z(t)$, $B_{Z^\ast}(t)$ and $B_\Psi(t)$ the Brier curves
of a probabilistic classifier $Z$, of its calibrated version $Z^\ast = P[A\,|\,\sigma(Z)]$,
and of the posterior probability $\Psi = P[A\,|\,\mathcal{H}]$ respectively, it holds that
\begin{equation}
B_Z(t) \ \ge \ B_{Z^\ast}(t) \ \ge \ B_\Psi(t), \qquad t \in [0,1].
\end{equation}
\item By Proposition~\ref{pr:area}, the Brier score for $Z$ equals twice the area below the Brier curve
of $Z$.
\item If both conditional distributions $P[Z\in \cdot |\,A]$ and $P[Z\in \cdot |\,A^c]$
have continuous Lebesgue densities $g_A$ and $g_{A^c}$ respectively, then the Brier curve $B(t)$ 
of $Z$ is continuously differentiable in $(0,1)$ with
\begin{equation}\label{eq:derivative}
\frac{d B}{d\, t}(t) \ = \ 1 - P[A] - P[Z\le t] + (1-t)\,P[A]\,g_A(t) - t\,(1-P[A])\,g_{A^c}(t).
\end{equation}
\end{itemize}

The following proposition not only applies to the `full information' posterior probability
$\Psi = P[A\,|\,\mathcal{H}]$, but also to any calibrated probabilistic classifier 
$Z = P[A\,|\,\sigma(Z)]$ (case $\mathcal{G} = \sigma(Z)$).
\begin{proposition}[Properties of Brier curves for calibrated probabilistic classifiers]
\label{pr:properties}
Under Assumption~\ref{as:setting}, let $\mathcal{G}$ denote some sub-$\sigma$-field of $\mathcal{H}$
and define $\Psi = P[A\,|\,\mathcal{G}]$. Then the Brier curve $t \mapsto B(t)$ 
of $\Psi$ according to Definition~\ref{de:curves}
has the following properties:
\begin{enumerate} 
\item \label{it:one} $B$ is concave and continuous on $[0,1]$, with $B(0) = 0$.
\item \label{it:two} For all $0 \le t \le 1$, it holds that
\begin{equation}\label{eq:CohenGeneralised}
\min(t, 1-t)\,E[\Psi\,(1-\Psi)] \ \le\ B(t) \ \le \ E[\Psi\,(1-\Psi)].
\end{equation}
\item \label{it:three} In \eqref{eq:CohenGeneralised}, we have 
$\min(t, 1-t)\,E[\Psi\,(1-\Psi)] = B(t)$
if and only if $\Psi$ is constant with $P[A] = 0 = \Psi$ or $P[A] = 1 = \Psi$. On the
right-hand side of \eqref{eq:CohenGeneralised}, it holds that
$B(t) = E[\Psi\,(1-\Psi)]$ 
if and only if $\Psi$ is constant with $P[A] = t = \Psi$.
\item \label{it:derivative} $B$ has for each $t\in[0,1)$ the right derivative 
$\frac{d^+}{d\,t}B(t) = 1  - P[A] - P[\Psi \le t]$ and
for each $t\in(0,1]$ the left derivative $\frac{d^-}{d\,t}B(t)= 1  - P[A] - P[\Psi < t]$.
\item \label{it:max} We call, for a real-valued random variable $Z$ and $\alpha \in (0,1)$, 
any number $z$ with 
$P[Z<z] \le \alpha \le P[Z\le z]$ an $\alpha$-quantile of $Z$. Then for all $(1-P[A])$-quantiles
 $t^\ast$ of $\Psi$
it holds that 
$$\max\limits_{0\le t\le 1} B(t) =  B(t^\ast).$$ 
\item \label{it:transform} Assume there is another probability
measure $Q$ on $(\Omega, \mathcal{A})$ such that $Q$ and $P$ are related through
prior probability shift on $\mathcal{G}$, i.e.\ for some $p, q \in(0,1)$ it holds that
\begin{subequations}
\begin{gather*}
q = Q[A] \neq P[A] = p, \qquad \text{as well as}\\
Q[G\,|\,A] = P[G\,|\,A] \quad \text{and}\quad Q[G\,|\,A^c] = P[G\,|\,A^c], \qquad G \in \mathcal{G}.
\end{gather*}
Let $\Psi_Q = Q[A\,|\,\mathcal{G}]$ and $\Psi_P = P[A\,|\,\mathcal{G}]$, and define for
$t \in [0,1]$:
\begin{equation}
\begin{split}
B_Q(t) &\ =\ (1-t)\,Q[A\cap \{\Psi_Q \le t\}] + t\,Q[A^c\cap \{\Psi_Q > t\}], \quad \text{and}\\
B_P(t) &\ =\ (1-t)\,P[A\cap \{\Psi_P \le t\}] + t\,P[A^c\cap \{\Psi_P > t\}].
\end{split}
\end{equation}
Then it follows that
\begin{equation}\label{eq:transform}
\begin{split}
\frac{B_Q(t)}{(1-t)\,q} &\ = \ \frac{B_P(s)}{(1-s)\,p}, \quad 0<t<1,\\
\text{with}\quad s & \ = \ \frac{(1-q)\,p\,t}{(p-q)\,t + q\,(1-p)}.
\end{split}
\end{equation}
\end{subequations}
\end{enumerate}
\end{proposition}
See Appendix~\ref{se:properties} for a proof of Proposition~\ref{pr:properties}.

\begin{itemize} 
\item \eqref{eq:CohenGeneralised} generalises Theorem~1 of \cite{Cohen&Goldszmidt} from Bayes error 
to cost-weighted Bayes loss. 
 \cite{zhao2013beyond} presented a general entropy-inspired approach to inequalities for
the cost-sensitive error. \eqref{eq:CohenGeneralised} can also be derived from inequality~(31)
of \cite{zhao2013beyond}.
\item Compare Proposition~\ref{pr:properties}, item~\ref{it:derivative} to \eqref{eq:derivative}.
Item~\ref{it:derivative} applies in full generality under Assumption~\ref{as:setting} while for
\eqref{eq:derivative} the existence of continuous Lebesgue densities for the conditional
feature distributions is needed. If this condition is satisfied then the comparison implies
that (in the notation of \eqref{eq:derivative}) the following equation is necessary for
the probabilistic classifier to be calibrated:
\begin{equation}
(1-t)\,P[A]\,g_A(t)  \ = \ t\,(1-P[A])\,g_{A^c}(t), \qquad 0 < t < 1.
\end{equation}
\item Proposition~\ref{pr:properties}, item~\ref{it:transform} shows that Brier curves 
for true posterior class properties can easily be transformed into each other as long
as the underlying distributions $Q$ and $P$ differ only by prior probability shift. Hence,
in principle, for studying the properties of Brier curves for calibrated probabilistic 
classifiers, it suffices to look at the case $P[A] = 1/2$.
\item Proposition~\ref{pr:properties}, item~\ref{it:max} motivates Proposition~\ref{pr:cov} below 
with potentially useful bounds for the refinement loss related to $\Psi= P[A\,|\,\mathcal{G}]$.
\end{itemize}

\begin{remark}\label{rm:relation}\emph{%
In the setting of Proposition~\ref{pr:properties}, 
the problem to compute the value of the Brier curve for $\Psi$ at $t \in (0,1)$ (or equivalently 
the cost-weighted Bayes loss of \eqref{eq:Bstar}) can be 
treated as a problem to determine an (unweighted) Bayes error, by making use of the following
easy to prove relation:
\begin{subequations}
\begin{equation}\label{eq:rel1}
\begin{split}
B(t) & \ = \ \bigl((1-t)\,P[A] + t\,(1-P[A])\bigr) 
    \min\limits_{G\in\mathcal{G}} \bigl(q\,P[G^c\,|\,A] + 
    (1-q)\,P[G\,|\,A^c] \bigr), \ \text{with}\\
 q & \ = \ \frac{(1-t)\,P[A]}{(1-t)\,P[A] + t\,(1-P[A])}.   
\end{split}
\end{equation}
Conversely, also the unweighted Bayes error can be determined by solving an optimisation problem
for a cost-weighted mean loss (or determining a value on a Brier curve):
\begin{equation}\label{eq:rel2}
\min\limits_{G\in\mathcal{G}} \bigl(P[A \cap G^c] + P[A^c \cap G] \bigr)
    \ = \ 2\,B_Q\bigl(1-P[A]\bigr).
\end{equation}
\end{subequations}
Here, the notation $B_Q$ indicates that the value of the Brier curve is determined
for a probability measure $Q$ such that the prior probability of $A$ under $Q$ is $Q[A] = 1/2$. 
} \hfill \qed
\end{remark}

\begin{remark} \label{rm:ROC}\emph{%
By making use of \eqref{eq:rel1} and \eqref{eq:rel2}, one can deploy 
Theorem~21 of \cite{reid2011information} to derive a representation of the Brier curve for 
$\Psi$ as in Proposition~\ref{pr:properties} in terms of the power function $\beta(\alpha)$
of the Neyman-Pearson test at size $\alpha$ based on the ratio of the densities of $P[G\,|\,A]$ and $P[G\,|\,A^c]$, 
$G\in\mathcal{G}$, and vice versa:
\begin{subequations}
\begin{equation}
\begin{split}
B(t) & \ = \ \bigl((1-t)\,P[A] + t\,(1-P[A])\bigr) 
    \min\limits_{\alpha \in[0,1]} \bigl((1-q)\,\alpha + q\, (1-\beta(\alpha))\bigr), \quad \text{with}\\
 q & \ = \ \frac{(1-t)\,P[A]}{(1-t)\,P[A] + t\,(1-P[A])},
\end{split}
\end{equation}
and
\begin{equation}
\beta(\alpha)  \ = \ \inf\limits_{p \in (0,1]} 
    \frac{(1-p)\,\alpha + p  - 2\,B_Q(1-p)}{p},
\end{equation}
\end{subequations}
where the notation $B_Q$ indicates that the value of the Brier curve is determined
for a probability measure $Q$ such that the prior probability of $A$ under $Q$ is $Q[A] = 1/2$
(as for \eqref{eq:rel2}).
See 
Section~6 of \cite{reid2011information} for more information on the relationship between
risk curves for costs, so-called risk curves for priors, and ROC curves.
} \hfill \qed
\end{remark} 

The lower bound for the refinement loss component of the Brier score (see Proposition~\ref{pr:brier} above)
in the following result is inspired by Proposition~\ref{pr:properties}, item~\ref{it:max} and is
a consequence of Proposition~\ref{pr:properties}, item~\ref{it:two}.

\begin{proposition} \label{pr:cov}
Under Assumption~\ref{as:setting}, denote
the probability of $A$ conditional on $\mathcal{H}$ by $\Psi = P[A\,|\,\mathcal{H}]$. 
Assume that there is a number $q \in [0,1]$ with $P[\Psi \le q] = 1-P[A]$. 
Then the following bounds apply for the refinement loss $E[\Psi\,(1-\Psi)]$:
\begin{equation}\label{eq:corr}
1 - \mathrm{corr}[\mathbf{1}_A, \,\mathbf{1}_{\{\Psi > q\}}] \ \le \
\frac{E[\Psi\,(1-\Psi)]}{P[A]\,(1-P[A])} \ \le \
1 - \mathrm{corr}[\mathbf{1}_A, \,\mathbf{1}_{\{\Psi > q\}}]^2.
\end{equation}
\end{proposition}
See Appendix~\ref{se:cov} for a proof of Proposition~\ref{pr:cov}.

In Proposition~\ref{pr:cov}, $\mathbf{1}_{\{\Psi > q\}}$ is a moment-matching estimator of $\mathbf{1}_A$ 
(same expected value
and variance). If $\Psi$ is not exactly known but there is a scoring classifier $S$ (i.e.\ a $\mathcal{H}$-measurable 
random variable) which is strongly comonotonic
with $\Psi$ then the event $\{\Psi > q\}$ can be replaced by $\{S > q^\ast\}$ where $q^\ast$ denotes
a $(1-P[A])$-quantile of $S$. Thus bounds for the refinement loss could be determined without 
exact knowledge of $\Psi$.

\subsection{Application}
\label{se:application}

In this section, based on the observations made before, we suggest 
\begin{itemize} 
\item what can be done to avoid or at least control the grouping (or information gap) loss when
estimating posterior class probabilities, and
\item how this can be supported by the use of Brier curves.
\end{itemize}
\begin{figure}
\begin{center}
\includegraphics[width=10cm]{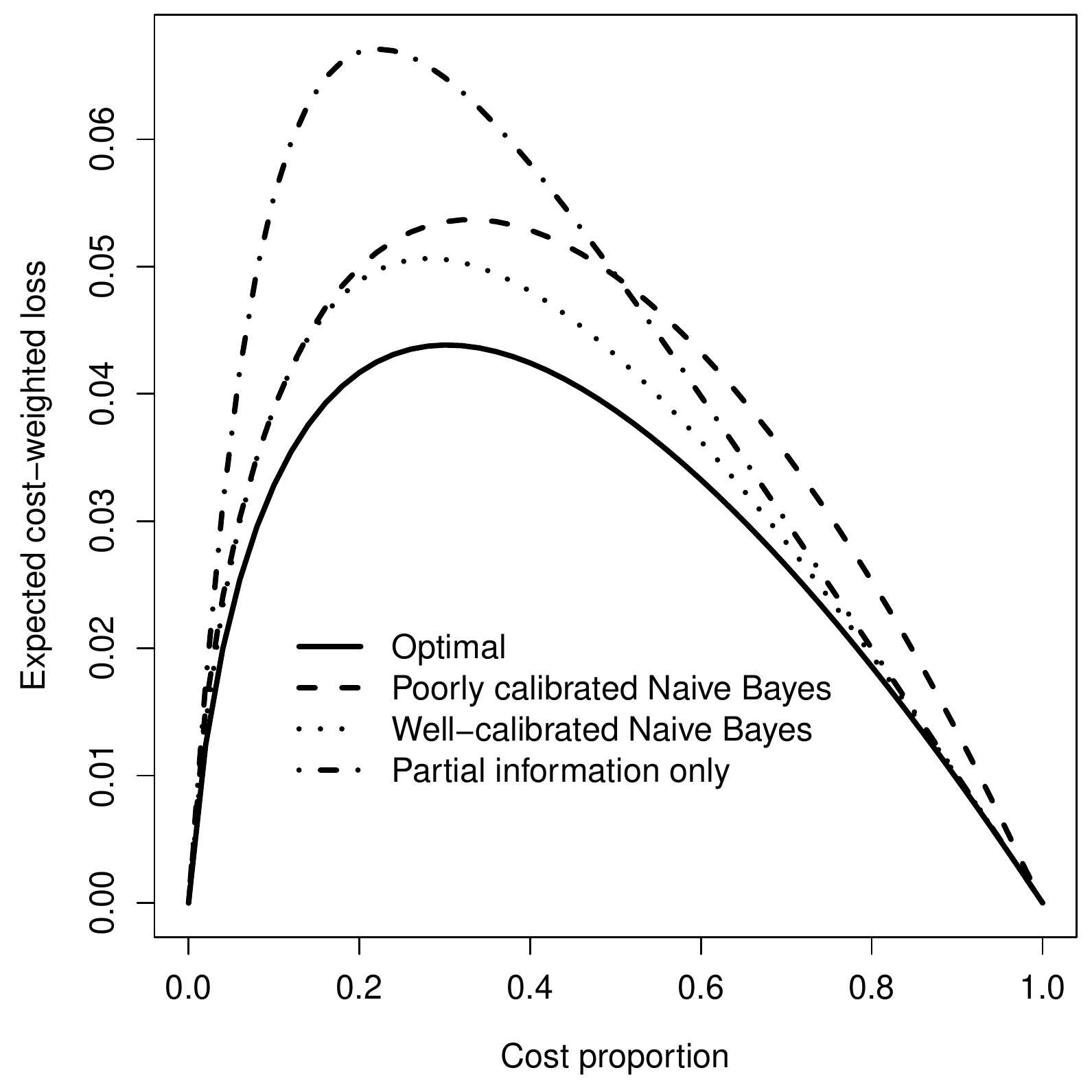}
\caption{{\small{}Examples of Brier Curves. The prior probability of the positive class is
10\%. See 
Section~\ref{se:application} for explanations of the different curves.}\label{fig:Brier}}
\end{center}
\end{figure}
In Figure~\ref{fig:Brier}, Brier curves are used to illustrate the decomposition of the
Brier score as defined in Proposition~\ref{pr:brier} and the concept of grouping loss in particular.
\begin{itemize} 
\item For Figure~\ref{fig:Brier}, a model based on bivariate 
normal feature distributions with different mean
vectors but identical covariance matrices has been chosen. 
The two components of the normal distributions are positively correlated.
\item The solid curve shows the Brier curve of the true posterior class probability, making use of
the full information available (i.e.\ both marginal components and their correlation). The area under this curve
is half of the optimal Brier score for the model and, therefore, also equals half of the refinement
loss (or irreducible loss) associated with the model.
\item The dash-dotted curve shows the Brier curve of a `partial' posterior class probability, making use
of only part of the available information (the first component of the normal distributions). The 
area between this curve and the solid curve equals half of the calibration loss, i.e.\ the Brier
score for the partial posterior class probability minus the refinement loss for the model.
In this case, the calibration loss is identical to the grouping loss because the partial posterior
class probability is perfectly calibrated given the available information in the shape of the first
component only.
\item The dashed curve shows the Brier curve of a misspecified posterior class probability
where the same formulae have been applied as for the true `full' posterior class probability but
with correlation set to zero. This corresponds to the naive Bayes approach to the estimation of the
posterior class probability and results in a poorly calibrated estimate. In this case, both
components of the calibration loss as described in \eqref{eq:refined} are positive.
\item The dotted curve shows the Brier curve of the perfectly calibrated posterior class probability
resulting from the naive Bayes approach. While the correlation here is taken into account for the
calibration, its misspecification in the design of the naive Bayes score still results in an
information gap (grouping loss) that is reflected in the area between the dotted curve and the solid
curve. The area between the dotted and dashed curves is the calibration loss component due
to using a misspecified formula for calibrating the naive Bayes score underlying the dashed curve (this
is called `group-wise calibration loss' in \citealp{kull2014reliability}).
\end{itemize}

Recall that in the context of approach~\ref{it:uni} of Section~\ref{se:calibration} to
the estimation of posterior class probabilities we are dealing with a hierarchy of scoring 
classifiers (under Assumption~\ref{as:setting}):
\begin{itemize} 
\item Scoring classifier \citep{hernandez2011brier}: 
$\mathcal{H}$-measurable random variable $S: \Omega \to \mathbb{R}$.
\item Probabilistic classifier \citep{hernandez2011brier}: 
$\mathcal{H}$-measurable random variable $Z: \Omega \to [0,1]$,
interpreted as estimator of posterior class probability $\Psi = P[A\,|\,\mathcal{H}]$.
\item Calibrated probabilistic classifier: $\mathcal{H}$-measurable random variable $Z: \Omega \to [0,1]$
with the property $Z = P[A\,|\,\sigma(Z)]$. \cite{hernandez2012unified}
called in Definition~38 such probabilistic classifiers `perfectly calibrated'.
\item True posterior class probability $\Psi = P[A\,|\,\mathcal{H}]$: In the terms
introduced before in this paper, it can also be described as a sufficient, calibrated 
and probabilistic classifier.
\end{itemize}
Typically, a scoring classifier $S$ will be the outcome of some kind of optimisation procedure
in the course of which one tries to link $S$ as best as possible to the positive class $A$, on
the basis of a pre-defined set of features. This procedure commonly is described as `learning 
a classifier' and can result in a scoring classifier with an arbitrary range (like in the case of
SVMs) or in a probabilistic classifier (like in the case of binary logistic regression).

A number of efficient methods were proposed in the literature to deal with the problem of
transforming a scoring or probabilistic classifier into a calibrated classifier (see 
\citealp{kull2017betacalibration}, and the references therein). However, focussing on this
step only would fail to control the grouping loss.

\subsection{Recommendation} 
\label{se:recom}
Theorem~\ref{th:devroye} shows that the optimisation criterion for
learning a scoring classifier $S$ must be carefully chosen in order to make sure that $S$ becomes sufficient for 
the full information content of $\mathcal{H}$ with respect to $A$. Only then the grouping 
loss for the estimation of the true posterior class probability will vanish. Proposition~\ref{pr:increasing}
suggests a way to achieve this: Show that $S$ satisfies \eqref{eq:exact} if not for all $t \in (0,1)$
then at least for a large number of $t$ that exhausts the whole interval $(0,1)$. This exercise 
can be supported by inspection of the Brier curve for $S$ constructed by plotting the 
right-hand side of \eqref{eq:exact} against $t$. This Brier curve should be concave and dominated
by the Brier curves of all other candidate scoring classifiers. 

By Remark~\ref{rm:ROC}, instead of striving to find the minimal Brier curve, equivalently we could try
to find a scoring classifier $S$ such that the associated ROC curve is maximal (see, e.g., 
\citealp{MarroccoROC2008}). This observation provides additional support for the
`ranking loss + isotonic regression' approach proposed by \cite{menon2012predicting}.

Recall also that 
Theorem~\ref{th:Langford} provides a promising alternative approach 
for combining optimal classifiers to obtain well-performing estimators of the posterior class probability.
See \cite{langford2005estimating} for a practical implementation of this approach. What makes
this approach particularly attractive is the fact that the classifiers being combined need not belong
to the same type (e.g.\ SVMs) but can be chosen at discretion.


\section{Conclusions}
\label{se:conclusions}

We have pointed out `grouping loss' as a potentially contributing factor to the 
miscalibration of probabilistic classifiers. Grouping loss is caused by a gap
between the information available for the calibration and the information actually taken into account
for the calibration exercise. The absence of grouping loss is equivalent to the property
of sufficiency as defined in Section~32.3 of \cite{devroye1996probabilistic} and known from statistics and 
the literature on dimension reduction techniques.
Sufficiency in turn can be characterised in terms of the information needed for optimally
solving a range of cost-weighted mean loss problems. We have presented further criteria 
for sufficiency like
comonotonicity which might be easier to establish. In addition, we have presented an example
that shows that `nesting' as part of the definition of sufficiency is necessary for fully
benefiting from the concept.

Nonetheless, in practice it could be difficult or even impossible to identify 
a sufficient scoring classifier. In this case, the so-called probing reduction 
\citep{langford2005estimating} might be useful. We have restated the theorem that
justifies this approach in order to better describe its connection to sufficiency
and grouping loss. 

Finally we have revisited the so-called Brier curves \citep{reid2011information, 
hernandez2011brier} and provided an extensive list
of their properties. Thanks to the fact that the area under a Brier curve for a probabilistic
classifier is just half of the Brier score of that classifier, Brier curves are a useful
tool for analysing calibration questions and can complement or even replace ROC curves
in this respect as well as a criterion for classifier development.

\section*{Acknowledgements}

The author thanks Tilmann Gneiting and an anonymous reviewer for suggestions that
helped to improve earlier versions of this paper. 


\addcontentsline{toc}{section}{References}

\appendix

\section{Proofs}
\label{se:proofs}

\subsection{Proof of Theorem~\ref{th:Langford}} 
\label{se:Langford}

For the readers' convenience, we present here an expanded version 
of the proof by \cite{langford2005estimating}. This version, in particular,
makes the informal part of the original proof between Eq.~(2) and the following inequality
more precise.

The first equation in \eqref{eq:Langford} 
is obtained through a combination of Proposition~\ref{pr:brier} and 
Proposition~\ref{pr:area}.

For all $H \in \mathcal{H}$ and $t \in (0,1)$, it can easily be shown that
$$L(H,t) \ = \ (1-t)\,P[A] + \int_H (t - \Psi)\,dP.$$
Define the symmetric difference $M \triangle N$ of two sets $M$ and $N$ as
$$M \triangle N = M\setminus N \cup N \setminus M.$$
Then, since $L_{\mathcal{H}}^\ast(t) = L(\{\Psi > t\}, t)$ by 
Theorem~32.4 of \cite{devroye1996probabilistic}, it follows for all $t \in (0,1)$ that
\begin{equation*}
L(H(t), t) - L_{\mathcal{H}}^\ast(t) = \ \int_{H(t) \triangle \{\Psi>t\}} |\Psi - t|\,dP.
\end{equation*}
Hence, by Fubini's theorem it follows that
\begin{align*}
\int_0^1 L(H(t), t) - L_{\mathcal{H}}^\ast(t)\, dt & = 
\int_0^1 \int \mathbf{1}_{\{(t,\omega):h(t,\omega)=1,\, \Psi(\omega)\le t\}}\,(t-\Psi(\omega))\,
 P(d\omega)\,dt\\
& \  +  \int_0^1 \int\mathbf{1}_{\{(t,\omega):h(t,\omega)=0,\, \Psi(\omega)> t\}}\,(\Psi(\omega) - t)\,
 P(d\omega)\,dt\\
& =  \int \left(\int_0^1 \mathbf{1}_{\{t:h(t,\omega)=1,\, \Psi(\omega)\le t\}}\,(t-\Psi(\omega))\,
 dt\right) P(d\omega)\\
& \ + \int \left(\int_0^1 \mathbf{1}_{\{t:h(t,\omega)=0,\, \Psi(\omega)> t\}}\,(\Psi(\omega)-t)\,
 dt\right) P(d\omega).
\end{align*}
Define $$g(t,\omega)\ =\ \begin{cases}
1, & \text{if}\ Z(\omega) > t,\\
0, & \text{if}\ Z(\omega) \le t.
\end{cases}$$
With this notation, it follows that $\{Z > t\} = \{\omega:g(t,\omega)=1\}$ and therefore
\begin{align*}
\int_0^1 L(\{Z > t\}, t) - L_{\mathcal{H}}^\ast(t)\, dt & = \int \left(\int_0^1 
\mathbf{1}_{\{t:g(t,\omega)=1,\, \Psi(\omega)\le t\}}\,(t-\Psi(\omega))\,
 dt\right) P(d\omega)\\
& \ + \int \left(\int_0^1 \mathbf{1}_{\{t:g(t,\omega)=0,\, \Psi(\omega)> t\}}\,(\Psi(\omega)-t)\,
 dt\right) P(d\omega),
\end{align*}
Hence, \eqref{eq:Langford} is implied if we can show that for all $\omega \in \Omega$
\begin{multline}\label{eq:toshow}
\int_0^1 \mathbf{1}_{\{t:h(t,\omega)=1,\, \Psi(\omega)\le t\}}\,(t-\Psi(\omega))\,
 dt + \int_0^1 \mathbf{1}_{\{t:h(t,\omega)=0,\, \Psi(\omega)> t\}}\,(\Psi(\omega)-t)\,
 dt \\
\ \ge \ \int_0^1 
\mathbf{1}_{\{t:g(t,\omega)=1,\, \Psi(\omega)\le t\}}\,(t-\Psi(\omega))\,
 dt +  \int_0^1 \mathbf{1}_{\{t:g(t,\omega)=0,\, \Psi(\omega)> t\}}\,(\Psi(\omega)-t)\,
 dt.
\end{multline}
Let $I = \{t:g(t,\omega)=0,\,h(t,\omega)=1\}$ and $J = \{t:g(t,\omega)=1,\,h(t,\omega)=0\}$.
The definition of $g$ implies
\begin{equation}\label{eq:gh}
\int_0^1 g(t, \omega)\,dt \ = \ Z(\omega) \ = \ \int_0^1 h(t, \omega)\,dt.
\end{equation}
Denoting by $\ell(I)$ and $\ell(J)$ the Lebesgue measures of $I$ and $J$ respectively,
from \eqref{eq:gh} follows
$$\ell(I) \ = \ \ell(J).$$
The definition of $I$ and $J$ moreover implies that \eqref{eq:toshow} is equivalent to
\begin{multline}\label{eq:toshow2}
\int_{\Psi(\omega)}^1 \mathbf{1}_I(t)\,(t-\Psi(\omega))\,
 dt + \int_0^{\Psi(\omega)} \mathbf{1}_J(t)\,(\Psi(\omega)-t)\,
 dt \\
\ \ge \ \int_{\Psi(\omega)}^1 \mathbf{1}_J(t)\,(t-\Psi(\omega))\,
 dt + \int_0^{\Psi(\omega)} \mathbf{1}_I(t)\,(\Psi(\omega)-t)\,
 dt.
\end{multline}
Consider now the case $\Psi(\omega) < Z(\omega)$. Then 
\begin{gather*}
[Z(\omega), 1] \ = \ [Z(\omega), 1]\cap [\Psi(\omega), 1] \ = \ \{t:g(t,\omega)=0\} \cap [\Psi(\omega), 1],\\
[\Psi(\omega), Z(\omega)) \ = \ [0, Z(\omega)) \cap [\Psi(\omega), 1] \ = \ 
\{t:g(t,\omega)=1\} \cap [\Psi(\omega), 1]
\\ \text{and}\quad 
\emptyset \ = \ [Z(\omega), 1] \cap [0, \Psi(\omega)) \ = \
\{t:g(t,\omega)=0\} \cap [0, \Psi(\omega)).
\end{gather*}
Making use of these observations and of $I \subset \{t:g(t,\omega)=0\}$ and
$J \subset \{t:g(t,\omega)=1\}$, we obtain for the terms in \eqref{eq:toshow2}
\begin{align*}
\int_{\Psi(\omega)}^1 \mathbf{1}_I(t)\,(t-\Psi(\omega))\,
 dt & \ \ge \ \bigl(Z(\omega) - \Psi(\omega)\bigr)\,\ell(I),\\
\int_0^{\Psi(\omega)} \mathbf{1}_J(t)\,(\Psi(\omega)-t)\,
 dt & \ \ge\ 0\\
\int_{\Psi(\omega)}^1 \mathbf{1}_J(t)\,(t-\Psi(\omega))\,
 dt  & \ \le \  \bigl(Z(\omega) - \Psi(\omega)\bigr)\,\ell(J),\\
\int_0^{\Psi(\omega)} \mathbf{1}_I(t)\,(\Psi(\omega)-t)\,
 dt & \ = \ 0. 
\end{align*}
Since $\ell(I) = \ell(J)$ it follows that \eqref{eq:toshow2} and therefore also \eqref{eq:toshow}
are true in the case $\Psi(\omega) < Z(\omega)$.

Assume now $\Psi(\omega) \ge Z(\omega)$. Then
\begin{gather*}
[0, Z(\omega)) \ = \ [0, Z(\omega)) \cap [0, \Psi(\omega)) \ = \ \{t:g(t,\omega)=1\} \cap [0, \Psi(\omega)),\\
\emptyset \ = \ [0, Z(\omega)) \cap [\Psi(\omega), 1] \ = \
\{t:g(t,\omega)=1\} \cap [\Psi(\omega), 1]\\
\text{and} \quad [Z(\omega), \Psi(\omega)) \ = \ [Z(\omega), 1] \cap [0, \Psi(\omega)) \ = \ 
\{t:g(t,\omega)=0\} \cap [0, \Psi(\omega)).
\end{gather*}
Making use of these further observations and of $I \subset \{t:g(t,\omega)=0\}$ and
$J \subset \{t:g(t,\omega)=1\}$, we obtain in this case for the terms in \eqref{eq:toshow2}
\begin{align*}
\int_{\Psi(\omega)}^1 \mathbf{1}_I(t)\,(t-\Psi(\omega))\,
 dt & \ \ge \ 0,\\
\int_0^{\Psi(\omega)} \mathbf{1}_J(t)\,(\Psi(\omega)-t)\,
 dt & \ \ge\ \bigl(\Psi(\omega) - Z(\omega)\bigr)\,\ell(J)\\
\int_{\Psi(\omega)}^1 \mathbf{1}_J(t)\,(t-\Psi(\omega))\,
 dt  & \ = \  0,\\
\int_0^{\Psi(\omega)} \mathbf{1}_I(t)\,(\Psi(\omega)-t)\,
 dt & \ \le \ \bigl(\Psi(\omega) - Z(\omega)\bigr)\,\ell(I). 
\end{align*}
Since $\ell(I) = \ell(J)$ it again follows that \eqref{eq:toshow2} and therefore also \eqref{eq:toshow}
are true in the case $\Psi(\omega) \ge Z(\omega)$. This completes the proof. \hfill $\Box$

\subsection{Proof of Proposition~\ref{pr:properties}}
\label{se:properties}

\begin{subequations}
Observe that for fixed $t \in [0,1]$ and all $x \in[0,1]$, it holds that
\begin{equation}\label{eq:x}
\min(t, 1-t)\,x\,(1-x) \ \le \ \min\bigl((1-t)\,x,\, t\,(1-x)\bigr)\ \le \ x\,(1-x).
\end{equation}
From the properties of conditional probabilities, it follows that
\begin{equation}\label{eq:written}
B(t) \ = \ E\bigl[\min\bigl((1-t)\,\Psi,\,t\,(1-\Psi)\bigr)\bigr].
\end{equation}
\eqref{eq:x} and \eqref{eq:written} together imply items~\ref{it:one}, \ref{it:two} and \ref{it:three}.
\end{subequations}

\textbf{Item~\ref{it:derivative}.} For $0 \le t \le 1$, observe that
\begin{align*}
B(t) &\ = \ (1-t)\,E[\Psi\,\mathbf{1}_{\{\Psi\le t\}}] + 
        t\,E[(1-\Psi)\,\mathbf{1}_{\{\Psi> t\}}]\\
 & \ = \ E\bigl[\min(t, \Psi)\bigr] - t\, P[A].       
\end{align*}
As a consequence of concavity, for each $t \in (0,1]$ the left derivative $\frac{d^-}{d t}B(t)$  
and for each $t \in [0,1)$ the right derivative $\frac{d^+}{d t}B(t)$
of $B$ in $t$ exist and are finite. For fixed $\omega \in \Omega$, it holds that
\begin{align*}
\frac{\partial^+}{\partial t} \min(t, \Psi(\omega)) & \ = \ \mathbf{1}_{\{t<\Psi\}}(\omega),\\
\frac{\partial^-}{\partial t} \min(t, \Psi(\omega)) & \ = \ \mathbf{1}_{\{t\le \Psi\}}(\omega).
\end{align*}
By the dominated convergence theorem, this implies item~\ref{it:derivative}.

\textbf{Item \ref{it:max}.} Because $B$ is concave, continuous and non-negative on $[0,1]$ we know 
that it has a single maximum
that is assumed by at least one $t^\ast \in (0,1)$.
All such $t^\ast$ must satisfy
\begin{gather*}
\frac{d^-}{d t}B(t^\ast)\ \ge\ 0\ \ge\ \frac{d^+}{d t}B(t^\ast)\\
\iff 1  - P[A] - P[\Psi < t^\ast] \ \ge\ 0\ \ge\ 1  - P[A] - P[\Psi \le t^\ast]\\
\iff P[\Psi < t^\ast] \ \le \ 1 -P[A] \ \le \ P[\Psi \le t^\ast].
\end{gather*}
Hence the $(1-P[A])$-quantiles of $\Psi$ are the maximisers of $B$. This proves item~\ref{it:max}.

\textbf{Item~\ref{it:transform}.} By Eq.~(2.4) of \cite{saerens2002adjusting}, $\Psi_Q$ and
$\Psi_P$ are related by
\begin{equation*}
\Psi_Q \ = \ \frac{\frac{q}{p}\,\Psi_P}{\frac{q}{p}\,\Psi_P + \frac{1-q}{1-p}\,(1-\Psi_P)}.
\end{equation*}
This implies 
$$\Psi_P \ \le s \quad \iff \quad \Psi_Q \ \le \ t,$$
and hence 
$$P[\Psi_P \le s\,|\,A] \ = \ Q[\Psi_Q \le t\,|\,A] \qquad \text{and} \qquad
P[\Psi_P > s\,|\,A^c] \ = \ Q[\Psi_Q > t\,|\,A^c].$$ 
Furthermore, by definition of $s$, it holds that
$$\frac{s\,(1-p)}{(1-s)\,p} \ = \ \frac{t\,(1-q)}{(1-t)\,q}.$$
This completes the proof of \eqref{eq:transform}. \hfill \qed

\subsection{Proof of Proposition~\ref{pr:cov}} 
\label{se:cov}

From Proposition~\ref{pr:properties} it follows 
\begin{equation}\label{eq:first}
(1-q)\,P[A\cap \{\Psi \le q\}] + q\,P[A^c \cap \{\Psi > q\}] \ = \ L(\{\Psi > q\}, q) 
\ \le \ E[\Psi\,(1-\Psi)].
\end{equation}
By assumption, we have $P[\Psi > q] = P[A]$ which implies that
\begin{eqnarray*}
\lefteqn{(1-q)\,P[A\cap \{\Psi \le q\}] + q\,P[A^c \cap \{\Psi > q\}]}\hspace{3cm}\\ 
&\ = \ &
(1-q)\,E[\Psi\,\mathbf{1}_{\{\Psi \le q\}}] + q\,E[(1-\Psi)\,\mathbf{1}_{\{\Psi > q\}}]\\
& \ = \ & E[\Psi\,\mathbf{1}_{\{\Psi \le q\}}] - q\,E[\Psi\,\mathbf{1}_{\{\Psi \le q\}}] +
q\,P[A] - q\,E[\Psi\,\mathbf{1}_{\{\Psi > q\}}] \\
& \ = \ & E[\Psi\,\mathbf{1}_{\{\Psi \le q\}}] + q\,(P[A] - E[\Psi])\\
& \ =\ & E[\Psi\,\mathbf{1}_{\{\Psi \le q\}}]\\
& \ = \ & P[A] - P[A\cap\{\Psi > q\}].
\end{eqnarray*}
Combining this with \eqref{eq:first}, we obtain
\begin{gather*}
P[A] - P[A\cap\{\Psi > q\}] \ \le \ E[\Psi\,(1-\Psi)]\\
\iff\quad  P[A](1-P[A]) - \mathrm{cov}[\mathbf{1}_A,\,\mathbf{1}_{\{\Psi > q\}}] 
    \ \le \ E[\Psi\,(1-\Psi)].
\end{gather*}
Dividing both sides of this last inequality by $P[A](1-P[A])$ gives the left-hand side of \eqref{eq:corr}.

Regarding the right-hand side of \eqref{eq:corr}, observe that because of the
$\mathcal{H}$-measurability of $\Psi$ it holds that
\begin{equation}\label{eq:linear}
\begin{split}
E[(\mathbf{1}_A - \Psi)^2] &\ =\ \min\limits_{Z\ \mathcal{H}\text{-measurable}} 
    E[(\mathbf{1}_A - Z)^2] \\
    & \ \le\ \min\limits_{a, b \in \mathbb{R}} 
E[(\mathbf{1}_A - (a\,\mathbf{1}_{\{\Psi > q\}}+b))^2].
\end{split}
\end{equation}
In plain language, the minimum squares error for approximating $\mathbf{1}_A$ by
general regression is lower than minimum squares error for approximating $\mathbf{1}_A$ by
linear regression on $\mathbf{1}_{\{\Psi > q\}}$. It is well known (or follows from a short
calculation) that the minimising linear regression coefficient $a$ is given by
$$a\ =\ \frac{\mathrm{cov}[\mathbf{1}_A,\,\mathbf{1}_{\{\Psi > q\}}]}
    {\mathrm{var}[\mathbf{1}_{\{\Psi > q\}}]} 
\ =\ \mathrm{corr}[\mathbf{1}_A,\,\mathbf{1}_{\{\Psi > q\}}].$$
The optimising $b$ is obtained as
$$b \ = \  P[A]\,(1 - \mathrm{corr}[\mathbf{1}_A,\,\mathbf{1}_{\{\Psi > q\}}]).$$
In addition, it holds for these numbers $a$ and $b$ that
\begin{align*}
E[(\mathbf{1}_A - (a\,\mathbf{1}_{\{\Psi > q\}}+b))^2]&  \ = \
    \mathrm{var}[\mathbf{1}_A] - \mathrm{var}[a\,\mathbf{1}_{\{\Psi > q\}}+b] \\
    & \ = \ P[A](1-P[A])  - \mathrm{corr}[\mathbf{1}_A,\,\mathbf{1}_{\{\Psi > q\}}]^2 \,
    \mathrm{var}[\mathbf{1}_{\{\Psi > q\}}].
\end{align*}
Since $E[(\mathbf{1}_A - \Psi)^2] = E[\Psi\,(1-\Psi)]$, the right-hand side 
inequality of \eqref{eq:corr} now follows from \eqref{eq:linear}.
 \hfill $\Box$

\end{document}